\documentclass[a4paper]{article}
\usepackage{graphicx} % Required for inserting images
\usepackage{authblk}

\title{Learning linear acyclic causal model including Gaussian noise using ancestral relationships}
\author[1]{Ming Cai}
\author[1]{Penggang Gao}
\author[2]{Hisayuki Hara}

\affil[1]{Graduate School of Informatics, Kyoto University}
\affil[2]{Institute for Liberal Arts and Sciences, Kyoto University}

\usepackage[reqno]{amsmath}%
\usepackage{amssymb}%
\usepackage{bm}
\usepackage[table]{xcolor}%
\allowdisplaybreaks

\usepackage{amsthm}

\usepackage{titlesec}%
\titlelabel{\thetitle. }%
\usepackage{xcolor}%
\usepackage{mleftright}%
\usepackage{xspace}%
\usepackage{graphicx}
\usepackage{subcaption} % 用于创建子图环境
\usepackage{hyperref}%

\usepackage{algorithm}
\usepackage{algpseudocode}
\usepackage{makecell}

\theoremstyle{plain}%
\newtheorem{theorem}{Theorem}[section]

\newtheorem{lemma}[theorem]{Lemma}

\newtheorem{corollary}[theorem]{Corollary}

\newtheorem{proposition}[theorem]{Proposition}

\theoremstyle{plain}%
\newtheorem{definition}[theorem]{Definition}

\numberwithin{figure}{section}%
\numberwithin{table}{section}%
\numberwithin{equation}{section}%

\newcommand{\indep}{\mathop{\perp\!\!\!\!\perp}}
\newcommand{\notindep}{\mathop{\not \perp\!\!\!\!\perp}}

\makeatletter
\newenvironment{breakablealgorithm}
  {% \begin{breakablealgorithm}
   \begin{center}
     \refstepcounter{algorithm}% New algorithm
     \hrule height.8pt depth0pt \kern2pt% \@fs@pre for \@fs@ruled
     \renewcommand{\caption}[2][\relax]{% Make a new \caption
       {\raggedright\textbf{\fname@algorithm~\thealgorithm} ##2\par}%
       \ifx\relax##1\relax % #1 is \relax
         \addcontentsline{loa}{algorithm}{\protect\numberline{\thealgorithm}##2}%
       \else % #1 is not \relax
         \addcontentsline{loa}{algorithm}{\protect\numberline{\thealgorithm}##1}%
       \fi
       \kern2pt\hrule\kern2pt
     }
  }{% \end{breakablealgorithm}
     \kern2pt\hrule\relax% \@fs@post for \@fs@ruled
   \end{center}
  }
\makeatother

\usepackage[inline]{enumitem}

\newlist{compactenumA}{enumerate}{5}%
\setlist[compactenumA]{topsep=0pt,itemsep=-1ex,partopsep=1ex,parsep=1ex,%
   label=(\Alph*)}%

\newlist{compactenuma}{enumerate}{5}%
\setlist[compactenuma]{topsep=0pt,itemsep=-1ex,partopsep=1ex,parsep=1ex,%
   label=(\alph*)}%

\newlist{compactenumI}{enumerate}{5}%
\setlist[compactenumI]{topsep=0pt,itemsep=-1ex,partopsep=1ex,parsep=1ex,%
   label=(\Roman*)}%

\newlist{compactenumi}{enumerate}{5}%
\setlist[compactenumi]{topsep=0pt,itemsep=-1ex,partopsep=1ex,parsep=1ex,%
   label=(\roman*)}%

\newlist{compactitem}{itemize}{5}%
\setlist[compactitem]{topsep=0pt,itemsep=-1ex,partopsep=1ex,parsep=1ex,%
   label=\ensuremath{\bullet}}%

\usepackage{tikz}
\usetikzlibrary{shapes,decorations}

\begin{document}
\maketitle

\begin{abstract}
This paper discusses algorithms for learning causal DAGs. 
The PC algorithm \cite{Spirtes} makes no assumptions other than the faithfulness to the causal model and can identify only up to the Markov equivalence class. 
LiNGAM \cite{Shimizu2006} assumes linearity and continuous non-Gaussian disturbances for the causal model, and the causal DAG defining LiNGAM is shown to be fully identifiable. 
The PC-LiNGAM \cite{hoyer2008b}, a hybrid of the PC algorithm and LiNGAM, can identify up to the distribution-equivalence pattern of a linear causal model, even in the presence of Gaussian disturbances. However, in the worst case, the PC-LiNGAM has factorial time complexity for the number of variables.

In this paper, we propose an algorithm for learning the distribution-equivalence patterns of a linear causal model with a lower time complexity than PC-LiNGAM, using the causal ancestor finding algorithm in Maeda and Shimizu \cite{Maeda2020}, which is generalized to account for Gaussian disturbances.
\end{abstract}

\section{Introduction}
Learning the causal structure among high-dimensional variables is a fundamental challenge across various disciplines. 
We are often forced to learn causal structures based solely on observed data.
Over the past quarter century, theoretical research on causal discovery based on observational data has made remarkable progress, and many practical algorithms have been proposed. 

In this paper, we assume that the causal structure is defined by a directed acyclic graph (DAG). 
We also assume that the causal model has no latent confounders. 
When learning a causal DAG nonparametrically, we need to focus on conditional independence (CI) relationships between variables. 
In general, however, multiple causal DAGs may exist such that the CI relationships among variables are identical. 
The set of causal DAGs encoding the same CI relationship among variables is called the Markov equivalence class (MEC). 

The PC algorithm \cite{Spirtes} can identify a causal DAG up to a MEC using the CI relationships under the faithfulness assumption for the causal model. 
The greedy equivalence search (GES) \cite{Chickering2002} learns a causal DAG using a model criterion such as BIC. 
The GES can also identify a causal DAG up to a MEC. 
These algorithms cannot identify the orientation of some edges in a causal DAG. 
A graph in which edges in the causal DAG whose orientations cannot be determined are replaced by undirected edges is called a d-separation-equivalence pattern (DSEP, e.g., \cite{hoyer2008b}). 
The output of the PC algorithm and the GES is obtained as a DSEP. 

Identifying a causal DAG beyond a MEC requires additional assumptions to the causal model. 
Shimizu et al. \cite{Shimizu2006} assumed that the causal model is linear and that disturbances are independently distributed, continuous, and non-Gaussian. 
Such a model is called the linear non-Gaussian acyclic model (LiNGAM). 
Shimizu et al. \cite{Shimizu2006} showed that the causal DAG that defines LiNGAM is fully identifiable using the independent component analysis (ICA, e.g., \cite{Hyvarinen2002}). 
Their algorithm is called the ICA-LiNGAM. 
While the time complexity of the PC algorithm and the GES is exponential for the dimension of variables, the ICA-LiNGAM can estimate a causal DAG with polynomial time complexity for the dimension of variables. 
Since the advent of the ICA-LiNGAM, much work has been devoted to improving the LiNGAM estimation algorithm and generalizing the model. 

Hoyer et al. \cite{hoyer2008b} proposed the PC-LiNGAM for identifying a causal DAG when the linear causal model includes Gaussian disturbances. 
The PC-LiNGAM is a hybrid of the PC algorithm and the ICA-LiNGAM. 
Identifying the entire causal DAG may be impossible if a linear causal model includes Gaussian disturbances. 
When two different linear causal models have the same joint distribution, they are called distribution-equivalent. The distribution-equivalent linear causal models are represented by a distribution-equivalence pattern (\cite{hoyer2008b}, DEP) consisting of a graph with both directed and undirected edges. % and a vector of binary variables that indicates whether each disturbance is Gaussian or not. 
The causal DAG defining a linear causal model containing Gaussian disturbances can only be identified up to a DEP. The PC-LiNGAM can identify a DEP for the linear causal model.  

The PC-LiNGAM first estimates a DSEP that encodes a MEC using the PC algorithm. Next, it finds the DAG that maximizes the ICA objective function among the MEC. In the worst case, when a DSEP is an undirected complete graph, this procedure is factorial time for the dimension of the variables, making it challenging to implement in high-dimensional cases. 

In this paper, we assume that the causal model is linear and faithful, and we propose a new algorithm that outputs a DEP given a DSEP. The proposed method uses the ancestor-finding algorithm proposed by Maeda and Shimizu \cite{Maeda2020}, which assumes non-Gaussian disturbances. In this paper, we generalize it to the case where the linear causal model contains Gaussian disturbances and use it to determine the orientation of undirected edges in a DSEP. 
We can show that the proposed method generically identifies a causal DAG up to the DEP of a causal DAG. 
In this paper, "generically" means except the set of measure zeros in the parameter space. 
We can also show that the proposed method has a polynomial time complexity on the dimension of the variables. 

The rest of this paper is organized as follows: Section 2 summarizes some existing causal structure learning algorithms
and clarifies the position of the proposed method. Section 3 describes the details of the proposed algorithm. 
Section 4 confirms the usefulness of the proposed methods through computer experiments. 
Section 5 concludes the paper. The proofs of theorems and lemmas are provided in the Appendix. 

\subsection{Terminologies and notations on graphs}
Before going to Section 2, this subsection summarizes some terminologies and notations for graphs used in this paper. 

In this paper, we assume that the causal graph is a DAG. A DSEP and a DEP are graphs with both directed and undirected edges, as described above.
A graph with directed and undirected edges is called a mixed graph. A mixed graph without a directed cycle is called a chain graph (e.g., \cite{Andersson1997}). 
Both a DSEP and a DEP are chain graphs.

A directed graph $G$ is said to be weakly connected if the undirected graph obtained by replacing all directed edges in $G$ with undirected edges is connected.
For a directed graph, a maximal set of vertices that induces a weakly connected subgraph is called the weakly connected component.
In this paper, we identify a weakly connected component with the subgraph induced by the weakly connected component. 

For a graph $G=(\bm{X},E)$, the subgraph induced by $\bm{X}' \subset \bm{X}$ is denoted as $G(\bm{X}') = (\bm{X}', \bm{X}' \times \bm{X}' \cap E)$. 

Furthermore, directed and undirected edge between $x_i$ and $x_j$, $i \ne j$ are denoted as $x_i \to x_j$ and $x_i - x_i$, respectively. 
Both directed and undirected edges are sometimes identified with the set $\{x_i,x_j\}$ consisting of two variables. That is, for a weakly connected graph $G=(\bm{X},E)$, $\bigcup_{e\in E} e = \bm{X}$ holds. 
We represent a V-structure as $x_i \rightarrow x_k \leftarrow x_j$ and a directed cycle as $x_i \to x_j \to x_k \to x_i$. 
For simplicity, we abuse $x_i \rightarrow x_k \leftarrow x_j \in G$ to indicate that the V-structure is contained in $G$.

\section{Related studies}
This section reviews several statistical causal discovery methods relevant to this paper: the PC algorithm \cite{sg1991,Spirtes}, LiNGAM \cite{Shimizu2006, Shimizu2011}, PC-LiNGAM \cite{hoyer2008b}, and the ancestor finding in RCD \cite{Maeda2020}.
\subsection{PC algorithm}
\label{sec:pc}
The PC algorithm \cite{sg1991, Spirtes} is a constraint-based algorithm for estimating causal DAGs using only the CI relationships among variables. 
The PC algorithm can only identify a causal DAG up to a MEC. 
Let $\bm{X}=(x_1,\ldots,x_p)^\top$ be a $p$-variate random vector. 
%and let $G=(X,E)$ be the causal DAG for $\bm{X}$. 
The PC algorithm starts with an undirected complete graph on $p$-variables in $\bm{X}$.
For all pair $(x_i,x_j)$, $i<j$, if there is $S \subset \bm{X} \setminus \{x_i,x_j\}$ satisfying $x_i \indep x_j \mid S$, remove the undirected edge $x_i - x_j$. 
$S$ is called a separating set (sepset) of $(x_{i}, x_{j})$.
Then, we obtain the causal skeleton $G_{\mathrm{skel}}=(\bm{X},E_{\mathrm{skel}})$ of a causal DAG. 
Next, we find some V-structures in $G_{\mathrm{skel}}$ using the CI relationships between variables (Algorithm \ref{alg:PC1}) and then find a DSEP that encodes a MEC using the orientation rule by Verma and Pearl \cite{vp1992} and Meek \cite{meek1995} (Algorithm \ref{alg:PC2}). 

In the worst case, the PC algorithm requires $p(p-1)\cdot 2^{p-3}$ CI tests to obtain a causal skeleton, rendering the algorithm infeasible for high-dimensional datasets. 
There have been several previous studies aimed at improving the computation efficiency of the PC algorithm. 
Kalisch and B\"uhlmann \cite{kb2007}  proposed an algorithm for learning sparse Gaussian causal models using a PC algorithm with high computational efficiency. 
Giudice et al. \cite{Giudice2023} proposed the dual PC algorithm that searches a sepset from both zero-order and full-order and showed that their algorithm outperforms the PC algorithm in terms of computation time and estimation accuracy. 
A parallel computing technique for PC algorithms has also been proposed in Le et al. \cite{fastPC2016}. 
%%%%% 
%%%%%citations about PC-algorithm
%%%%%

Algorithm \ref{alg:PC1} is also executable even when latent confounders are present in the causal model. 
Fast Causal Inference (\cite{Spirtes}, FCI) generalizes the PC algorithm to the case where the causal model has latent confounders. 
\begin{breakablealgorithm} 
	\caption{Finding V-structure} 
	\label{alg:PC1} 
	\begin{algorithmic}[1]
		\Require A causal skeleton $G_{\mathrm{skel}}$ 
		\Ensure A chain graph with some V-structures $G_v$
        \ForAll{$x_i$ and $x_j$ such that $(x_i,x_j) \notin E_{\mathrm{skel}}$}
        % \Comment{For all unshielded triples}
            \If{$x_i - x_k - x_j \in E_{\mathrm{skel}}$ and $x_i \notindep x_j \mid x_k \cup S$}
                \State Orient $x_i \rightarrow x_k \leftarrow x_j$
            \EndIf
        \EndFor
    \end{algorithmic}
\end{breakablealgorithm}

\begin{algorithm}
	\caption{The orientation rule \cite{vp1992, meek1995}} 
	\label{alg:PC2} 
	\begin{algorithmic}[1]
		\Require A chain graph $G_v$
		\Ensure d-separation-equivalence pattern $G_{\mathrm{dsep}}$
        \ForAll{$x_i$ and $x_j$ such that $(x_i,x_j) \in E_{\mathrm{skel}}$}
            \If{$x_k \rightarrow x_i - x_j$}
                \State Orient $x_i \rightarrow x_j$ %\Comment{R1}
            \EndIf
            \If{$x_i \rightarrow x_k \rightarrow x_j$}
                \State Orient $x_i \rightarrow x_j$
                %\Comment{R2}
            \EndIf
            \If{$x_i - x_k \rightarrow x_j$, $x_i - x_l \rightarrow x_j$, and $(x_{k}, x_{l}) \notin E_{\mathrm{skel}}$}
                \State Orient $x_i \rightarrow x_j$
                %\Comment{R3}
            \EndIf
            \If{$x_i - x_k \rightarrow x_l$ and $x_k \rightarrow x_l \rightarrow X_j$}
                \State Orient $x_i \rightarrow x_j$
                %\Comment{R4}
            \EndIf
        \EndFor
    \end{algorithmic}
\end{algorithm}

\subsection{LiNGAM and variants}
\label{sec:LiNGAM}
Shimizu et al. \cite{Shimizu2006} considered identifying the entire causal DAG by imposing additional assumptions on the causal model.
Let $\bm{X}=(x_1,\ldots,x_p)^\top$ be a $p$-variate random vector.
In the following, we identify $\bm{X}$ with the variable set. 
They considered an acyclic linear structural equation model 
\begin{align}
    \label{model:LiNGAM}
        \bm{X}  =  B\bm{X} + \bm{\epsilon},   
\end{align}
where the disturbances $\bm{\epsilon}=(\epsilon_1,\ldots,\epsilon_p)^\top$ are independently distributed as continuous 
non-Gaussian distributions and the coefficient matrix $B$ can be transformed into a strictly lower triangular matrix by permuting the rows and columns. 
Shimizu et al. \cite{Shimizu2006} called the model (\ref{model:LiNGAM}) the linear non-Gaussian acyclic model (LiNGAM). 

Consider the reduced form of LiNGAM
%LiNGAM is rewritten by 
\begin{align}
    \label{model:LiNGAM2}
    \bm{X} = (I-B)^{-1} \bm{\epsilon},
\end{align}
where $I$ denotes the $p \times p$ identity matrix. 
Shimizu et al. \cite{Shimizu2006} showed that a causal DAG defining LiNGAM (\ref{model:LiNGAM}) is fully identifiable, using the fact that the model (\ref{model:LiNGAM2}) can be considered as the independent component analysis (ICA) model and provided an algorithm to estimate a causal DAG. 
The algorithm is known as the ICA-LiNGAM. 
Since the advent of the ICA-LiNGAM, much work has been devoted to improving the algorithm and generalizing the model.

As the dimension of the variables increases, the ICA-LiNGAM tends to converge to a locally optimal solution, resulting in lower estimation accuracy for small samples. 
To overcome this problem, Shimizu et al. \cite{Shimizu2011} proposed the DirectLiNGAM, which estimates LiNGAM using linear regressions. 

Hoyer et al. \cite{Hoyer2008} and Zhang and Hyv\"arinen \cite{Zhang2012} generalized LiNGAM to nonlinear and showed that the causal DAG for the nonlinear causal model is identifiable even if the disturbances are Gaussian. 
Vector autoregressive LiNGAM (VAR-LiNGAM) \cite{hyvarinen10a} is also a variant of LiNGAM to handle time series data.

Tashiro et al. \cite{tashiro2014parcelingam} and Hoyer et al. \cite{hoyer2008estimation} proposed feasible causal discovery methods considering latent confounders under the LiNGAM framework.
Maeda and Shimizu \cite{Maeda2020} proposed the repetitive causal discovery (RCD) that is intended to be applied to LiNGAM with latent confounders. 
RCD first determines the ancestral relationships between variables using linear regressions and independence tests, then creates a list of ancestor sets for each variable. 
The parent-child relationships between variables are determined from the CI relationships among variables in each estimated ancestor set. 

Divide-and-conquer algorithms have also been proposed to increase the feasibility of the DirectLiNGAM even when the sample size is smaller than the dimension of the observed variable. In these algorithms, variables are grouped into several subsets, the DirectLiNGAM is applied to each group, and the results are merged to estimate the entire causal DAG. 
Cai et al. \cite{Cai2013} and Zhang et al. \cite{Zhang2020} proposed algorithms for grouping variables based on the CI relationships between variables. 
Recently, Cai and Hara \cite{cai2024learning} have proposed another algorithm for variable grouping inspired by the ancestor-finding in RCD \cite{Maeda2020}.

In the proposed method, the ancestor finding in RCD is generalized to the case where the model includes Gaussian disturbances.

%%%%%
%%%%% sth about CAG, the proposed method in Master's period.
%%%%%
\subsection{PC-LiNGAM}
\label{sec:PC-LiNGAM}
As described in Section \ref{sec:pc}, the constraint-based approach, such as the PC algorithm, has no restrictions on the model or distribution of variables but can identify a causal DAG only up to a MEC. On the other hand, LiNGAM and its variant in Section \ref{sec:LiNGAM}, by constraining the causal model to be linear and the distribution of disturbances to be continuous non-Gaussian, can identify more directed edges of a causal DAG than the PC algorithm. 

To combine both advantages of the PC algorithm and LiNGAM, Hoyer et al. proposed a hybrid method of these algorithms named PC-LiNGAM \cite{Hoyer2008}. 
The PC-LiNGAM assumes that the causal model is linear (\ref{model:LiNGAM}), but the disturbances' distributions can be arbitrary continuous distributions, including the Gaussian distribution. 
To identify such a causal DAG, we need to focus not only on the graph structure but also on whether the disturbance is Gaussian or non-Gaussian.
Hoyer et al. \cite{hoyer2008b} defined ngDAG as follows. 

\begin{definition}[Hoyer et al. \cite{hoyer2008b}]
    \label{def: ngDAG}
An ngDAG $(G, ng)$ is defined by a pair of causal DAG $G$ and a $p$-dimensional vector $ng$ consisting of binary variables that take one when the disturbance for each variable follows a Gaussian distribution and zero otherwise.
\end{definition}

Two causal models defined by different ngDAGs with the same joint distribution are called distribution-equivalent. 
A chain graph encoding distribution-equivalent ngDAGs is called a distribution-equivalence pattern (DEP).
A causal DAG of a linear causal model containing Gaussian disturbances can only be identified up to a DEP. 
Hoyer et al. \cite{hoyer2008b} showed that the PC-LiNGAM can identify a causal graph up to a DEP.

The PC-LiNGAM first estimates a DSEP of a true causal DAG and then generates all DAGs that are consistent with the DSEP. 
For the structural equation model defined by each DAG, test the Gaussianity of OLS residuals $r_1,\ldots,r_p$ for each variable to set $ng$ defined in Definition \ref{def: ngDAG}.
Next, calculate the score for each DAG using the ICA objective function 
\begin{align}
    \label{eq: ICA objective}
    U_{f} = \sum_{i=1}^p
    \left(
    E[|r_{i}|] - \sqrt{\frac{2}{\pi}}
    \right)
\end{align}
and select the highest-scoring DAG. 
If the highest-scoring DAG has directed edges such that the residuals for the variables at both ends are Gaussian, the highest-scoring DAG is modified into a chain graph by replacing the directed edges with undirected edges.
At this stage, at least one residual of the variable at each end of any directed edges of the chain graph is non-Gaussian. 
Finally, using the orientation rules in Algorithm \ref{alg:PC2} as in the PC algorithm, the PC-LiNGAM outputs a DEP. 

Figure \ref{fig: example for PC-LiNGAM} illustrates a procedure of the PC-LiNGAM.
Diamond nodes represent variables with non-Gaussian disturbances, and circle nodes represent variables with Gaussian disturbances.
In this example, the true DAG is a directed complete DAG with four variables, as shown in (a).
Since the PC algorithm does not identify the orientation of any edge, the DSEP forms an undirected complete graph, as shown in (b).
Let (c) be the DAG with the highest score, where the orientations of $x_3 \leftarrow x_4$ are reversed compared to the true DAG (a).
If all the results of the Gaussianity test of residuals are correct, then $x_3 \leftarrow x_4$ is replaced with $x_3 - x_4$.
On the other hand, although the disturbances of $x_1$ and $x_3$ are also Gaussian, the orientation rule in Algorithm \ref{alg:PC2} maintains the orientation.
(d) is the DEP of the output, which, in this case, shows that the correct DEP was returned.

\begin{figure}[th]
    \centering
    \scalebox{0.8}{
    \begin{tikzpicture} 
        \node[draw, circle] (x_1) at (0, 1) {$x_1$};
        \node[draw, diamond] (x_2) at (-1, 0) {$x_2$};
        \node[draw, circle] (x_3) at (1, 0) {$x_3$};
        \node[draw, circle] (x_4) at (0, -1) {$x_4$};
        \draw[thick, ->] (x_1) -- (x_2);
        \draw[thick, ->] (x_1) -- (x_3);
        \draw[thick, ->] (x_1) -- (x_4);
        \draw[thick, ->] (x_2) -- (x_3);
        \draw[thick, ->] (x_2) -- (x_4);
        \draw[thick, ->] (x_3) -- (x_4);
        \node[] (X) at (0, -2) {(a) true graph};
        
        \node[draw, circle] (x_1) at (4, 1) {$x_1$};
        \node[draw, diamond] (x_2) at (3, 0) {$x_2$};
        \node[draw, circle] (x_3) at (5, 0) {$x_3$};
        \node[draw, circle] (x_4) at (4, -1) {$x_4$};
        \draw[thick, -] (x_1) -- (x_2);
        \draw[thick, -] (x_1) -- (x_3);
        \draw[thick, -] (x_1) -- (x_4);
        \draw[thick, -] (x_2) -- (x_3);
        \draw[thick, -] (x_2) -- (x_4);
        \draw[thick, -] (x_3) -- (x_4);
        \node[] (X) at (4, -2) {\makecell{(b)DSEP}};
        
        \node[draw, circle] (x_1) at (8, 1) {$x_1$};
        \node[draw, diamond] (x_2) at (7, 0) {$x_2$};
        \node[draw, circle] (x_3) at (9, 0) {$x_3$};
        \node[draw, circle] (x_4) at (8, -1) {$x_4$};
        \draw[thick, ->] (x_1) -- (x_2);
        \draw[thick, ->] (x_1) -- (x_3);
        \draw[thick, ->] (x_1) -- (x_4);
        \draw[thick, ->] (x_2) -- (x_3);
        \draw[thick, ->] (x_2) -- (x_4);
        \draw[thick, <-] (x_3) -- (x_4);
        \node[] (X) at (8, -2) {\makecell{(c) highest-scoring DAG}};
        
        \node[draw, circle] (x_1) at (12, 1) {$x_1$};
        \node[draw, diamond] (x_2) at (11, 0) {$x_2$};
        \node[draw, circle] (x_3) at (13, 0) {$x_3$};
        \node[draw, circle] (x_4) at (12, -1) {$x_4$};
        \draw[thick, ->] (x_1) -- (x_2);
        \draw[thick, ->] (x_1) -- (x_3);
        \draw[thick, ->] (x_1) -- (x_4);
        \draw[thick, ->] (x_2) -- (x_3);
        \draw[thick, ->] (x_2) -- (x_4);
        \draw[thick, -] (x_3) -- (x_4);
        \node[] (X) at (12, -2) {\makecell{(d) DEP}};
    \end{tikzpicture} 
    }
    \caption{An example to illustrate the PC-LiNGAM when handling a directed complete DAG with four variables.
    }
    \label{fig: example for PC-LiNGAM}
\end{figure}
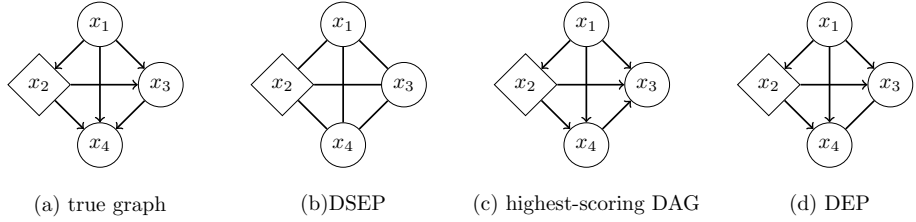

The PC-LiNGAM needs to calculate the objective function of the ICA for all DAGs that are consistent with a MEC. If the true causal DAG is a directed complete DAG, then the corresponding DSEP is an undirected complete graph. Then, any directed complete DAGs are consistent with this d-separation-equivalence pattern.
Thus, the number of graphs for which the objective function of ICA needs to be calculated is $p!$. In other words, the PC-LiNGAM is a factorial time algorithm in the worst case, making it infeasible for high-dimensional data.

\subsection{Ancestor finding in RCD}
\label{sec: rcd}
In the proposed method, the orientation of the undirected edges in a DSEP is determined by estimating the ancestral relationships between adjacent variables. 
Maeda and Shimizu \cite{Maeda2020} proposed an algorithm for estimating ancestral relationships between variables in LiNGAM that can be applied even in the presence of latent confounders. 
To determine the ancestral relationship between $x_i$ and $x_j$, $i \ne j$, consider the following pair of simple regression models,
\begin{equation}
\begin{aligned}
    \label{model:no ancestors}
    x_i &= \frac{\mathrm{Cov}(x_i,x_j)}{\mathrm{Var}(x_j)} x_j + r^{(j)}_i,\\
    x_j &= \frac{\mathrm{Cov}(x_i,x_j)}{\mathrm{Var}(x_i)} x_i + r^{(i)}_j, 
\end{aligned}
\end{equation}
where $r^{(j)}_i$ and $r^{(i)}_j$ are disturbances, and they are continuous and non-Gaussian according to the assumption of LiNGAM. 
Let $Anc_i$ denote the set of ancestors of $x_i$, and let $CA_{ij}$ be the set of common ancestors of $x_i$ and $x_j$. 

\begin{proposition}[Maeda and Shimizu \cite{Maeda2020}]
    \label{Prop:1}
    One of the following four conditions holds for the ancestral relationship between $x_i$ and $x_j$. 
    \begin{compactenumi}
        \item If $x_i \indep x_j$, then $x_i \notin Anc_j \wedge x_j \notin Anc_i$.
        \item If $x_i \indep r^{(i)}_j \wedge x_j \notindep r^{(j)}_i$, then $x_i \in Anc_j$.
        \item If $x_j \indep r^{(j)}_i \wedge x_i \notindep r^{(i)}_j$, then $x_j \in Anc_i$.
        \item If $x_i \notindep  r^{(i)}_j \wedge x_j \notindep r^{(j)}_i$, then $CA_{ij} \neq \emptyset$.
    \end{compactenumi}
    
\end{proposition}
For every pair of variables $x_i$ and $x_j$, we can check which of conditions (i) through (iv) in Proposition \ref{Prop:1} is satisfied.
The disturbances $r^{(j)}_i$ and $r^{(i)}_j$ are replaced with the OLS residuals for implementation. 

For all $(x_i,x_j)$ satisfying condition (iv), the determination of the ancestral relationship between $(x_i,x_j)$ is withheld. 
Assume that $x_i \in Anc_j$ or $x_i$ and $x_j$ have no ancestral relationship. 
If there exists $x_k \in CA_{ij}$ and there exists at least one backdoor path from $x_i$ to $x_j$ through $x_{k}$, 
we call $x_k$ a backdoor common ancestor of $x_i$ and $x_j$. 
Let $BCA_{ij}$ denote the set of backdoor common ancestors of $x_i$ and $x_j$.
Let $\mathcal{J}$ be the set of indices of the pair $(i,j)$ such that $(x_i,x_j)$ satisfies condition (iv). 
Then $CA_{ij}$ for $(i,j) \in \mathcal{J}$ always contains at least one backdoor common ancestor of $x_i$ and $x_j$. 

Let $CA^*_{ij}$ be the set of common ancestors of $(x_i,x_j)$ found during checking Proposition \ref{Prop:1} for all pairs of variables. 
In the following, $CA^*_{ij}$ is also considered as a vector. 
To remove the influence of $CA^*_{ij}$ on $x_i$ and $x_j$, consider the regression models,
\begin{equation}
    \begin{aligned}
        \label{model: no ancestors}
        x_i &= {CA^*_{ij}}^\top \cdot \bm{\alpha}_{ij} + v_i, \\
        x_j &= {CA^*_{ij}}^\top \cdot \bm{\alpha}_{ji} + v_j,   
    \end{aligned}
\end{equation}
where $\bm{\alpha}_{ij}$ and $\bm{\alpha}_{ji}$ are
\begin{align*}
\bm{\alpha}_{ij}&=
E\left[
CA_{ij}^* {CA^*_{ij}}^\top
\right]^{-1}
E\left[
CA^*_{ij} x_i
\right], \\    
\bm{\alpha}_{ji}&=
E\left[
CA_{ij}^* {CA^*_{ij}}^\top
\right]^{-1}
E\left[
CA^*_{ij} x_j
\right],
\end{align*}
respectively. 
If $\bm{X}$ follows LiNGAM, $v_i$ and $v_j$ are non-Gaussian disturbances. 
Furthermore, consider the following regression models for $v_i$ and $v_j$, 
\begin{equation}
    \begin{aligned}
        \label{eq: CA1}
        v_i =  \frac{\mathrm{Cov}(v_i,v_j)}{\mathrm{Var}(v_j)} v_j + u_i,\\
        v_j =  \frac{\mathrm{Cov}(v_i,v_j)}{\mathrm{Var}(v_i)} v_i + u_j,
    \end{aligned}
\end{equation}
where $u_i$ and $u_j$ are non-Gaussian disturbances. 
Then Proposition \ref{Prop:1} is generalized as follows. 
\begin{proposition}[Maeda and Shimizu \cite{Maeda2020}]
\label{Prop:2}
     One of the following four conditions holds for the ancestral relationship between $(x_i, x_j)$ for $(i,j) \in {\cal J}$. 
    \begin{compactenumi}
        \item If $v_i \indep v_j$, then $x_i \notin Anc_j \wedge x_j \notin Anc_i$.
        \item If $v_i \indep u_j \wedge v_j \notindep u_i$, then $x_i \in Anc_j$.
        \item If $v_j \indep u_j \wedge v_i \notindep u_j$, then $x_j \in Anc_i$.
        \item If $v_i \notindep u_j \wedge v_j \notindep u_i$, then $CA_{ij} \setminus CA^*_{ij} \ne \emptyset$. 
    \end{compactenumi}
\end{proposition}

Proposition \ref{Prop:1} is the case where $CA_{ij}=\emptyset$ and $\mathcal{J}=\emptyset$. 
For implementation, the disturbances $v$ and $u$ are replaced with OLS residuals. 
If $(x_i,x_j)$ satisfies condition (iv) of Proposition \ref{Prop:2}, the determination of the ancestral relationship between $x_i$ and $x_j$ is withheld. 
After checking Proposition \ref{Prop:2} for all $(x_i, x_j)$ such that $(i,j) \in \mathcal{J}$, update $\mathcal{J}$ to the set $(i,j)$ satisfying condition (iv) in Proposition \ref{Prop:2}. 
If $\mathcal{J} \ne \emptyset$, recheck Proposition \ref{Prop:2}. 
Theoretically, if the model contains no latent confounders, the procedure in Proposition \ref{Prop:2} can be repeated until $\mathcal{J} = \emptyset$ to completely determine the ancestral relationships of all $(x_i,x_j)$ in a causal DAG.

%%%
%%% Sec 3 Proposed
%%%
\section{Proposed Algorithm}
\label{sec: proposed alg}
This section introduces the proposed algorithm in detail.
We assume that the causal model is linear for 
$\bm{X}=(x_1,\ldots,x_p)^\top$ 
\begin{equation}
    \label{eq:SEM}
    \bm{X} = B \bm{X} + \bm{\epsilon}, 
\end{equation}
which is apparently the same as the model (\ref{model:LiNGAM}).
$B$ is a $p \times p$ matrix that can be transformed into a strictly lower triangular matrix by permuting the rows and columns. 
Let $G=(\bm{X},E)$ be the causal DAG that defines (\ref{eq:SEM}). 
As with the PC-LiNGAM, we assume that the disturbances
$\bm{\epsilon} = (\epsilon_1,\ldots,\epsilon_p)^\top$ are distributed as arbitrary continuous distributions, including the Gaussian distribution. 
We also assume the faithfulness assumption to the model (\ref{eq:SEM}). 

The proposed algorithm first obtains a DSEP by applying the PC algorithm to $\bm{X}$, as in the PC-LiNGAM. Then, it orients the undirected edges in the DSEP by estimating the ancestral relationships between the adjacent variable pairs in the DSEP. 

In the following, let $Pa_i$ and $Des_{i}$ denote the set of parents of $x_i$ and the set of descendants of $x_{i}$, respectively. 
The proofs of theorems in this section will be provided in the Appendix. 

\subsection{Ancestor finding in the model containing Gaussian disturbances}
In the proposed method, we determine the orientation of an undirected edge in a DSEP by estimating the ancestral relationship between two nodes connected by the undirected edge. 
We refer once again to the regression model in (\ref{model:no ancestors}). 
In the following, let $\mathcal{G}$ and $\mathcal{NG}$ denote Gaussian and non-Gaussian distributions, respectively. 
When the model is linear, the following theorem holds. 
\begin{theorem}
\label{ancestor finding simple}
Assume that $x_i \sim \mathcal{G}$ and $x_j \sim \mathcal{NG}$ and that $x_i$ and $x_j$ are adjacent in the true causal DAG $G = (\bm{X},E)$. 
Then, $x_i \to x_j \in E$. 
\end{theorem}
For two adjacent variables, one is Gaussian, and the other is non-Gaussian, their ancestral relationship is necessarily determined, with the Gaussian variable being the ancestor of the non-Gaussian variable. 
The next corollary follows directly from Theorem \ref{ancestor finding simple}.
\begin{corollary}
\label{ancestor finding simple corollary}
\label{cor: ancestor finding simple}
    For a variable $x_{i} \in \bm{X}$, 
    \begin{enumerate}
        \item $x_{i} \sim \mathcal{G}$ implies that $\forall x_{k} \in Anc_{i}, x_{k} \sim \mathcal{G}$.
        \item $x_{i} \sim \mathcal{NG}$ implies that $\forall x_{k} \in Des_{i}, x_{k} \sim \mathcal{NG}$.
    \end{enumerate} 
\end{corollary}

We generalize the ancestor-finding by Maeda and Shimizu \cite{Maeda2020} to the case where the model may contain Gaussian disturbances. 
We can obtain the following theorem, which generalizes Proposition \ref{Prop:1} to the case with Gaussian disturbances.

\begin{theorem}
    \label{ancestor finding}
    One of the following three conditions holds for the ancestral relationship between $x_i$ and $x_j$. 
    \begin{compactenumi}
        \item If $x_i \indep x_j$, $x_i \notin Anc_j$ and $x_j \notin Anc_i$. 
        \item $(x_i \notindep x_j) \wedge (x_i, x_j \sim \mathcal{NG}) \wedge (r^{(i)}_j \indep x_i) \wedge (r^{(j)}_i\notindep x_j) \Rightarrow (BCA_{ij} = \emptyset) \wedge (x_i \in Anc_j)$.
        \item $(x_i \notindep x_j )\wedge (x_i, x_j \sim \mathcal{NG}) \wedge (r^{(i)}_j \notindep x_i) \wedge (r^{(j)}_i\notindep x_j) \Rightarrow BCA_{ij} \neq \emptyset$.
    \end{compactenumi}
\end{theorem}

(ii) and (iii) in Theorem \ref{ancestor finding} focus on the non-Gaussianity of $x_i$ and $x_j$. Even without assuming that all the disturbances are non-Gaussian as in Maeda and Shimizu \cite{Maeda2020}, $x_i$ and $x_j$ could be non-Gaussian if some of the disturbances for their ancestors are non-Gaussian.
Therefore, Theorem \ref{ancestor finding} is a generalization of Proposition \ref{Prop:1} to the case where the model may contain a Gaussian disturbance.
Although Theorem \ref{ancestor finding} holds for any pair $x_i$ and $x_j$, the proposed method applies the theorem only to the two adjacent variables in a DSEP. In this case, $x_i \in Anc_j$ implies $x_i \to x_j \in E$. 
Since $x_i$ and $x_j$ are adjacent and hence dependent, the proposed method does not use the condition (i) in Theorem \ref{ancestor finding}. 
Under the conditions (ii) and (iii), if both $x_i$ and $x_j$ are non-Gaussian, it is determined that either $x_{i} \rightarrow x_{j}\in E$, $x_{i} \leftarrow x_{i} \in E$, or $BCA_{ij} \ne \emptyset$.
If only one of $x_{i}$ or $x_{j}$ is Gaussian, the orientation is determined by Theorem \ref{ancestor finding simple}.
If both $x_{i}$ and $x_{j}$ are Gaussian, the directions between $x_{i}$ and $x_{j}$ are not identifiable.  

Similarly to ancestor-finding in RCD in Section \ref{sec: rcd}, the determination of the orientation between $x_i$ and $x_j$ that satisfies condition (iii) in Theorem \ref{ancestor finding} is withheld. 
Let $BCA^*_{ij}$ denote the set of backdoor common ancestors of $x_i$ and $x_j$ found while checking the conditions (ii) and (iii) in Theorem \ref{ancestor finding}. 
$BCA^*_{ij}$ is also considered as a vector. 
As in the ancestor-finding in RCD, to remove the influence of $BCA^*_{ij}$ on $x_i$ and $x_j$, we consider the following linear regression model corresponding to (\ref{model: no ancestors}), 
\begin{equation}
    \begin{aligned}
        \label{model: no ancestors BCA star}
        x_i &= {BCA^{*}_{ij}}^\top \cdot \bm{\beta}_{ij} + v_i, \\
        x_j &= {BCA^{*}_{ij}}^\top \cdot \bm{\beta}_{ji} + v_j,
    \end{aligned}
\end{equation}
and the regression model for $v_i$ and $v_j$ (\ref{eq: CA1}). 
Then, corresponding to Theorems \ref{ancestor finding simple} and \ref{ancestor finding}, we can obtain the following Theorem. 
\begin{theorem}
\label{ancestor finding simple BCA star}
        Assume that $v_i \sim \mathcal{G}$ and $v_j \sim \mathcal{NG}$ and $x_i$ and $x_j$ are adjacent in $G$. 
    Then $x_i \to x_j \in E$. 
\end{theorem}
\begin{theorem}
    \label{ancestor finding BCA star}
    One of the following three conditions holds for the ancestral relationship between $x_i$ and $x_j$. 
    \begin{compactenumi}
%        \item $(v_i \sim \mathcal{G}, v_j \sim \mathcal{NG}) \Rightarrow (x_{i}, x_{j}) \in E$.
        \item $v_i \indep v_j \Rightarrow v_i \notin Anc_j \wedge v_j \notin Anc_i$.   
        \item $(v_i \notindep v_j) \wedge (v_i, v_j \sim \mathcal{NG}) \wedge (v_i \indep u_j) \wedge (v_j \notindep u_i) \Rightarrow x_{i} \in Anc_j \in E$.
        \item $(v_i \notindep v_j) \wedge (v_i, v_j \sim \mathcal{NG}) \wedge (v_j \notindep u_i) \wedge (v_i\notindep u_j) \Rightarrow BCA \setminus BCA^{*}_{ij} \neq \emptyset$.
\end{compactenumi}
\end{theorem}
Theorems \ref{ancestor finding simple} and \ref{ancestor finding} is the case where $BCA_{ij} = \emptyset$. 
For each adjacent pair $x_i$ and $x_j$, repeatedly check Theorem \ref{ancestor finding simple BCA star} and (ii) and (iii) in Theorem \ref{ancestor finding BCA star} until no further ancestral relationships or orientations can be identified. 
In the proposed method, Theorem \ref{ancestor finding BCA star} is also applied to adjacent $x_i$ and $x_j$, so $x_i \in Anc_j$ implies $x_i \to x_j \in E$. 

The proofs of Theorems \ref{ancestor finding simple BCA star} and \ref{ancestor finding BCA star} are much the same as the proofs of Theorems \ref{ancestor finding simple} and \ref{ancestor finding} and are omitted in the Appendix.

\subsection{Algorithm for finding DEP}
\label{sec:proposed}
In this section, we describe the details of the proposed algorithm based on the discussion in the previous subsection. The proposed algorithm is shown in Algorithm \ref{alg: proposed}.
In the algorithm, 
$G_{\mathrm{dsep}}=(\bm{X},E_{\mathrm{di}} \cup E_{\mathrm{ud}})$ denotes the DSEP of $G$, where $E_{\mathrm{di}}$ is the set of directed edges and $E_{\mathrm{ud}}$ is the set of undirected edges in $G_{\mathrm{dsep}}$. 
Let $G_{\mathrm{ud}} = (\bm{X}_{\mathrm{ud}},E_{\mathrm{ud}})$ be the undirected induced subgraph of $G_{\mathrm{dsep}}$, where 
$\bm{X}_{\mathrm{ud}} = \bigcup_{e \in E_{\mathrm{ud}}} e$. 
We note that $G_{\mathrm{ud}}$ is not necessarily connected. 
Let $G_{\mathrm{dep}}$ be the DEP of $G$.
In the following, BCA stands for a backdoor common ancestor. 

\begin{breakablealgorithm}
	\caption{Finding DEP based on ancestral relationship} 
	\label{alg: proposed} 
	\begin{algorithmic}[1]
		\Require $\bm{X}=(x_1,\ldots,x_p)^\top$
		\Ensure A DEP 
                $G_{\mathrm{dep}}=(\bm{X},\tilde{E}_{\mathrm{di}} \cup \tilde{E}_{\mathrm{ud}})$
            \State Apply PC algorithm to $\bm{X}$ and obtain a DSEP 
            $G_{\mathrm{dsep}}=(\bm{X},E_{\mathrm{di}} \cup E_{\mathrm{ud}})$ 
            \State $\tilde{E}_{\mathrm{di}} \leftarrow E_{\mathrm{di}}$, 
                   $\tilde{E}_{\mathrm{ud}} \leftarrow E_{\mathrm{ud}}$ 
            \ForAll{$x_i-x_j \in E_{\mathrm{ud}}$}
                \State $BCA^{*}_{ij} \leftarrow$ BCA of $\{x_i,x_j\}$ in $G_{\mathrm{dsep}}$ \Comment{Initialize $BCA^{*}_{ij}$}
            \EndFor
            \State Find all the connected components $\mathcal{C}$ of 
                $G_{\mathrm{ud}}=(\bm{X}_{\mathrm{ud}},E_{\mathrm{ud}})$% \par\hskip\algorithmicindent
                %\Comment{Find connected components}
            \ForAll{$G'=(\bm{X}',E'_{\mathrm{ud}}) \in \mathcal{C}$ such that $E'_{\mathrm{ud}} \ne \emptyset$}
                \State $E'_{\mathrm{di}}=\emptyset$
                \State Perform Gaussianity tests for each variable in $\bm{X}'$
                \State Split $\bm{X}'$ into Gaussian variables $\bm{X}^\prime_{\mathrm{ga}}$ and non-Gaussian variables $\bm{X}^\prime_{\mathrm{ng}}$
		      \If{$\bm{X}^\prime_{\mathrm{ng}} \ne \emptyset$}
                    \ForAll{$x_i-x_j \in E^\prime_{\mathrm{ud}}$} \Comment{Apply Theorem \ref{ancestor finding simple BCA star}}
                    \If{$x_i \in \bm{X}^\prime_{\mathrm{ga}}$ and $x_j \in \bm{X}^\prime_{\mathrm{ng}}$}
                        \State $\tilde{E}_{\mathrm{di}} \leftarrow \tilde{E}_{\mathrm{di}} \cup 
                        \{x_i \to x_j\}$, 
                               $\tilde{E}_{\mathrm{ud}} \leftarrow \tilde{E}_{\mathrm{ud}} \setminus \{x_i - x_j\}$
                        \State $E'_{\mathrm{di}} \leftarrow E^\prime_{\mathrm{di}} \cup \{x_i \rightarrow x_j\}$, 
                               $E'_{\mathrm{ud}} \leftarrow E'_{\mathrm{ud}} \setminus \{x_i - x_j\}$ 
                    \EndIf
                    \If{$x_i \in \bm{X}^\prime_{\mathrm{ng}}$ and $x_j \in \bm{X}^\prime_{\mathrm{ga}}$}
                        \State $\tilde{E}_{\mathrm{di}} \leftarrow \tilde{E}_{\mathrm{di}} \cup \{x_j \to x_i\}$, 
                               $\tilde{E}_{\mathrm{ud}} \leftarrow \tilde{E}_{\mathrm{ud}} \setminus \{x_j - x_i\}$
                        \State $E'_{\mathrm{di}} \leftarrow E^\prime_{\mathrm{di}} \cup \{x_j \rightarrow x_i\}$,  
                        $E'_{\mathrm{ud}} \leftarrow E'_{\mathrm{ud}} \setminus \{x_j - x_i\}$
                    \EndIf
                \EndFor
                \ForAll{$x_i-x_j \in E'_{\mathrm{ud}}$} \Comment{Update $BCA^{*}_{ij}$}
                    \State $BCA^{*}_{ij} \leftarrow BCA^{*}_{ij} \cup$ BCAs of $\{x_i,x_j\}$ in $G''=(\bm{X}',E'_{\mathrm{di}} \cup E'_{\mathrm{ud}})$
                \EndFor
                \State Find the induced subgraph $G'(\bm{X}^\prime_{\mathrm{ng}})$ 
                \State Find all the connected components $\mathcal{C}(\bm{X}^\prime_{\mathrm{ng}})$ of 
                $G^\prime(\bm{X}^\prime_{\mathrm{ng}})$
                \If {$|\mathcal{C}(\bm{X}^\prime_{\mathrm{ng}})| \ne 1$} %\Comment{Disconnected DMGs}
                    \State Append $\mathcal{C}(\bm{X}^\prime_{\mathrm{ng}})$ into $\mathcal{C}$, then go to line 10
                \EndIf
%                \STATE Perform Gaussianity tests for each element of $X'_{\mathrm{ng}}$
                \State {Flag $\leftarrow \mathrm{TRUE}$}
                \While{Flag}
                    \State {Flag $\leftarrow$ FALSE}
                    \ForAll{$x_i-x_j \in E'_{\mathrm{ud}}$} \Comment{Applying Theorem \ref{ancestor finding BCA star}}
                        \If{$BCA^{*}_{ij} \ne \emptyset$} 
                            \State {Regress $x_i$ and $x_j$ on $BCA^{*}_{ij}$ 
                            \par \hspace{\algorithmicindent} \hspace{\algorithmicindent} \hspace{\algorithmicindent}
                            and compute residuals $v_i$ and $v_j$}
                            \State $x_i \leftarrow v_i$, $x_j \leftarrow v_j$
                            \State $BCA^{*}_{ij} \leftarrow \emptyset$
                        \EndIf
    %                    \State Orient $x_i-x_j$ %\Comment{Theorem \ref{ancestor finding BCA star} (i)-(iii)}
                        \If{$x_i \to x_j$ is determined by (ii) in Theorem \ref{ancestor finding BCA star}}
                            \State $\tilde{E}_{\mathrm{di}} \leftarrow \tilde{E}_{\mathrm{di}} \cup \{x_i \to x_j\}$, 
                            $\tilde{E}_{\mathrm{ud}} \leftarrow \tilde{E}_{\mathrm{ud}} \setminus \{x_i - x_j\}$
                            \State $E'_{\mathrm{di}} \leftarrow E'_{\mathrm{di}} \cup\{x_i \rightarrow x_j\}$, 
                                   $E'_{\mathrm{ud}} \leftarrow E'_{\mathrm{ud}} \setminus \{x_i - x_j\}$
                        \ElsIf{$x_j \to x_i$ is determined by (ii) in Theorem \ref{ancestor finding BCA star}}
                            \State $\tilde{E}_{\mathrm{di}} \leftarrow \tilde{E}_{\mathrm{di}} \cup\{x_j \to x_i\}$, 
                            $\tilde{E}_{\mathrm{ud}} \leftarrow 
                            \tilde{E}_{\mathrm{ud}} \setminus \{x_j - x_i\}$
                            \State $E'_{\mathrm{di}} \leftarrow E'_{\mathrm{di}} \cup\{x_j \rightarrow x_i\}$, 
                                   $E'_{\mathrm{ud}} \leftarrow E'_{\mathrm{ud}} \setminus \{x_j - x_i\}$
                        \EndIf
                    \EndFor
                    \ForAll{$x_i-x_j \in E'_{\mathrm{ud}}$} \Comment{Update $BCA^{*}_{ij}$}
                        \State $BCA^{*}_{ij} \leftarrow$ BCA of $x_i$ and $x_j$ in $G''$ 
                    
                    \If{$BCA^{*}_{ij} \ne \emptyset$} 
                        \State {Flag $\leftarrow$ TRUE}
                        % \State Regress $x_i$ and $x_j$ on $BCA^{*}_{ij}$ and compute residuals $v_i$ and $v_j$
                        % % \par \hskip\algorithmicindent 
                        % \State $x_i \leftarrow v_i$, $x_j \leftarrow v_j$
                        % \State $BCA^{*}_{ij} \leftarrow \emptyset$
                    \EndIf  
                    \EndFor
                \EndWhile
                \State Find the induced subgraph 
                $G'(\bm{X}'_{\mathrm{ud}})=(\bm{X}'_{\mathrm{ud}},E'_{\mathrm{ud}})$, 
                \par \hskip\algorithmicindent 
                where 
                $\bm{X}'_{\mathrm{ud}}:= \bigcup_{e \in E'_{\mathrm{ud}}} e$
                \State Find all the connected components $\mathcal{C}(\bm{X}^\prime_{\mathrm{ud}})$ of 
                $G^\prime(\bm{X}^\prime_{\mathrm{ud}})$
                % \If {$|\mathcal{C}(\bm{X}^\prime_{\mathrm{ud}})| \ne 1$} %\Comment{Find connected components}
                \State Append $\mathcal{C}(\bm{X}^\prime_{\mathrm{ud}})$ into $\mathcal{C}$
                % \EndIf
            \EndIf
            \EndFor
            \State Apply the orientation rule to $G_{\mathrm{dep}}$ \\
            \Return $G_{\mathrm{dep}}$
	\end{algorithmic} 
\end{breakablealgorithm} 

The flow of Algorithm \ref{alg: proposed} is summarized below. 

After obtaining a $G_{\mathrm{dsep}}$ by the PC algorithm, the first step is to find its undirected subgraph $G_{\mathrm{ud}}$. Then, for adjacent $x_i$ and $x_j$ in $G_{\mathrm{ud}}$, initialize $BCA^{*}_{ij}$ to the set of BCA of $x_i$ and $x_j$ in $G_{\mathrm{dsep}}$. 
Line 6 finds the connected components ${\cal C}$ of $G_{\mathrm{ud}}$. %As shown in the Appendix, each connected component of the induced sub-DAG $G(\bm{X}_{\mathrm{ud}})$ has only one source node. 

Below line 7, orient the undirected edges of each connected component $G'=(\bm{X}', E'_{\mathrm{ud}})$ of $G_{\mathrm{ud}}$. 
Lines 12-24 orient undirected edges $x_{i} - x_{j} \in E'_{\mathrm{ud}}$ by using Theorem \ref{ancestor finding simple BCA star}, based on the results of Gaussianity tests for each variable in $\bm{X}'$ in line 9 and update $BCA_{ij}^*$. 
Lines 25-29 find the undirected subgraph of $G'$ induced by non-Gaussian variables. 
The Flag in line 30 is a binary variable that controls the loop starting from line 31.
Lines 31-53 use Theorem \ref{ancestor finding BCA star} to orient undirected edges in $G'$. 
If $x_i$ and $x_j$ satisfy (iii) in Theorem \ref{ancestor finding BCA star}, the determination of the orientation between them is withheld. 
Lines 47-52 update $BCA^{*}_{ij}$ and $\bm{X}'$ from the information of the newly identified directed edge.
If new directions are identified during Lines 39-45, the flag variable is updated to TRUE, which has been toggled to FALSE in line 32.
Lines 54 to 56 find the undirected subgraph of $G'$ and update ${\cal C}$. 
Finally, apply the orientation rule in Algorithm \ref{alg:PC2} in line 59 and return $G_{\mathrm{dep}}$. 

\begin{figure}[!ht]
    \centering
    \scalebox{0.8}{
    \begin{tikzpicture} 
        \node[draw, circle] (x_1) at (0, 1) {$x_1$};
        \node[draw, diamond, fill=gray] (x_2) at (-1, 0) {$x_2$};
        \node[draw, circle, fill=gray] (x_3) at (1, 0) {$x_3$};
        \node[draw, circle, fill=gray] (x_4) at (0, -1) {$x_4$};
        \draw[thick, ->] (x_1) -- (x_2);
        \draw[thick, ->] (x_1) -- (x_3);
        \draw[thick, ->] (x_1) -- (x_4);
        \draw[thick, ->] (x_2) -- (x_3);
        \draw[thick, ->] (x_2) -- (x_4);
        \draw[thick, ->] (x_3) -- (x_4);
        \node[] (X) at (0, -2) {(a) true DAG $G$};

        \node[] (s) at (1.5, 0) {};
        \node[] (e) at (2.5, 0) {};
        \draw[thick, ->] (s) -- (e);
        
        \node[draw, circle] (x_1) at (4, 1) {$x_1$};
        \node[draw, diamond, fill=gray] (x_2) at (3, 0) {$x_2$};
        \node[draw, circle, fill=gray] (x_3) at (5, 0) {$x_3$};
        \node[draw, circle, fill=gray] (x_4) at (4, -1) {$x_4$};
        \draw[thick, -] (x_1) -- (x_2);
        \draw[thick, -] (x_1) -- (x_3);
        \draw[thick, -] (x_1) -- (x_4);
        \draw[thick, -] (x_2) -- (x_3);
        \draw[thick, -] (x_2) -- (x_4);
        \draw[thick, -] (x_3) -- (x_4);
        \node[] (X) at (4, -2) {\makecell{(b) DSEP}};

        \node[] (s) at (5.5, 0) {};
        \node[] (e) at (6.5, 0) {};
        \draw[thick, ->] (s) -- (e);
        \node[] (X) at (6, 1) {\makecell{Based on \\ Theorem \\ \ref{ancestor finding simple BCA star}}};
        
        \node[draw, circle] (x_1) at (8, 1) {$x_1$};
        \node[draw, diamond, fill=gray] (x_2) at (7, 0) {$x_2$};
        \node[draw, circle, fill=gray] (x_3) at (9, 0) {$x_3$};
        \node[draw, circle, fill=gray] (x_4) at (8, -1) {$x_4$};
        \draw[thick, ->] (x_1) -- (x_2);
        \draw[thick, ->] (x_1) -- (x_3);
        \draw[thick, ->] (x_1) -- (x_4);
        \draw[thick, -] (x_2) -- (x_3);
        \draw[thick, -] (x_2) -- (x_4);
        \draw[thick, -] (x_3) -- (x_4);
        \node[] (X) at (8, -2) {\makecell{(c)}};

        \node[] (s) at (9.5, 0) {};
        \node[] (e) at (10.5, 0) {};
        \draw[thick, ->] (s) -- (e);
        \node[] (X) at (10, 1) {\makecell{Regress on \\ $x_{1}$}};
        
        \node[draw, circle] (x_1) at (12, 1) {$x_1$};
        \node[draw, diamond, fill=gray] (x_2) at (11, 0) {$x_2$};
        \node[draw, circle, fill=gray] (x_3) at (13, 0) {$x_3$};
        \node[draw, circle, fill=gray] (x_4) at (12, -1) {$x_4$};
        \draw[thick, dashed, ->] (x_1) -- (x_2);
        \draw[thick, dashed, ->] (x_1) -- (x_3);
        \draw[thick, dashed, ->] (x_1) -- (x_4);
        \draw[thick, -] (x_2) -- (x_3);
        \draw[thick, -] (x_2) -- (x_4);
        \draw[thick, -] (x_3) -- (x_4);
        \node[] (X) at (12, -2) {\makecell{(d) }};

        % \node[] (s) at (12, -2.5) {};
        % \node[] (m) at (12, -5) {};
        % \node[] (e) at (9.5, -5) {};
        \draw[thick, ->] (12, -2.5) -- (12, -5) -- (9.5, -5);
        \node[] (X) at (10.75, -4) {\makecell{Based on \\ Theorem \\ \ref{ancestor finding BCA star}}};

        % \node[draw, circle] (x_1) at (12, -4) {$x_1$};
        % \node[draw, diamond] (x_2) at (11, -5) {$x_2$};
        % \node[draw, circle] (x_3) at (13, -5) {$x_3$};
        % \node[draw, circle] (x_4) at (12, -6) {$x_4$};
        % \draw[thick, dashed, ->] (x_1) -- (x_2);
        % \draw[thick, dashed, ->] (x_1) -- (x_3);
        % \draw[thick, dashed, ->] (x_1) -- (x_4);
        % \draw[thick, -] (x_2) -- (x_3);
        % \draw[thick, -] (x_2) -- (x_4);
        % \draw[thick, -] (x_3) -- (x_4);
        % \node[] (X) at (12, -7) {\makecell{(e) }};

        % \node[] (s) at (10.5, -5) {};
        % \node[] (e) at (9.5, -5) {};
        % \draw[thick, ->] (s) -- (e);

        \node[draw, circle] (x_1) at (8, -4) {$x_1$};
        \node[draw, diamond, fill=gray] (x_2) at (7, -5) {$x_2$};
        \node[draw, circle, fill=gray] (x_3) at (9, -5) {$x_3$};
        \node[draw, circle, fill=gray] (x_4) at (8, -6) {$x_4$};
        \draw[thick, dashed, ->] (x_1) -- (x_2);
        \draw[thick, dashed, ->] (x_1) -- (x_3);
        \draw[thick, dashed, ->] (x_1) -- (x_4);
        \draw[thick, ->] (x_2) -- (x_3);
        \draw[thick, ->] (x_2) -- (x_4);
        \draw[thick, -] (x_3) -- (x_4);
        \node[] (X) at (8, -7) {\makecell{(f)}};

        \node[] (s) at (6.5, -5) {};
        \node[] (e) at (5.5, -5) {};
        \draw[thick, ->] (s) -- (e);
        \node[] (X) at (6, -4) {\makecell{Regress on \\ $BCA^{*}_{34}$}};
        
        \node[draw, circle] (x_1) at (4, -4) {$x_1$};
        \node[draw, diamond, fill=gray] (x_2) at (3, -5) {$x_2$};
        \node[draw, circle] (x_3) at (5, -5) {$x_3$};
        \node[draw, circle] (x_4) at (4, -6) {$x_4$};
        \draw[thick, dashed, ->] (x_1) -- (x_2);
        \draw[thick, dashed, ->] (x_1) -- (x_3);
        \draw[thick, dashed, ->] (x_1) -- (x_4);
        \draw[thick, dashed, ->] (x_2) -- (x_3);
        \draw[thick, dashed, ->] (x_2) -- (x_4);
        \draw[thick, -] (x_3) -- (x_4);
        \node[] (X) at (4, -7) {\makecell{(g) }};

        \node[] (s) at (2.5, -5) {};
        \node[] (e) at (1.5, -5) {};
        \draw[thick, ->] (s) -- (e);
        \node[] (X) at (2, -4) {\makecell{Apply Alg \ref{alg:PC2} \\ and output}};

        \node[draw, circle] (x_1) at (0, -4) {$x_1$};
        \node[draw, diamond] (x_2) at (-1, -5) {$x_2$};
        \node[draw, circle] (x_3) at (1, -5) {$x_3$};
        \node[draw, circle] (x_4) at (0, -6) {$x_4$};
        \draw[thick, ->] (x_1) -- (x_2);
        \draw[thick, ->] (x_1) -- (x_3);
        \draw[thick, ->] (x_1) -- (x_4);
        \draw[thick, ->] (x_2) -- (x_3);
        \draw[thick, ->] (x_2) -- (x_4);
        \draw[thick, -] (x_3) -- (x_4);
        \node[] (X) at (0, -7) {\makecell{(h) DEP}};
    \end{tikzpicture} 
    }
    \caption{An example to illustrate the proposed method when handling a directed complete DAG over four variables.
    }
    \label{fig: example for proposal}
\end{figure}
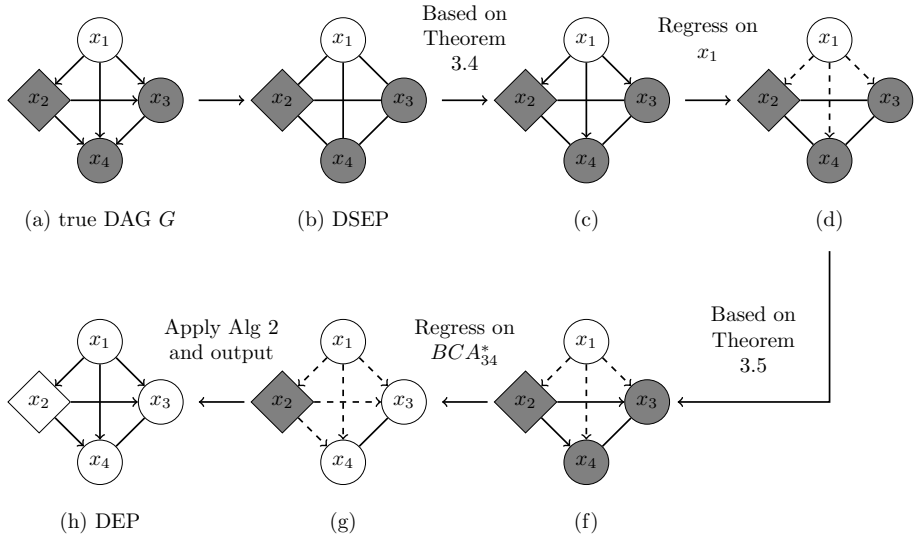

Figure \ref{fig: example for proposal} illustrates how the proposed method identifies a DEP. 
In Figure \ref{fig: example for proposal}, diamond nodes represent variables with non-Gaussian disturbances, while circle nodes represent variables with Gaussian disturbances.
Gray nodes represent non-Gaussian variables, and white nodes represent Gaussian variables. 
The dashed lines represent the directed edges that have been removed when finding induced subgraphs at lines 25 and 47 in Algorithm \ref{alg: proposed}. The true DAG (a) and its DSEP (b) are the same as in Figure \ref{fig: example for PC-LiNGAM} (a) and (b), respectively. 
Since $x_2$ has a non-Gaussian disturbance, $x_2$ is also non-Gaussian. Therefore, by Corollary \ref{cor: ancestor finding simple} $x_3$ and $x_4$ are also non-Gaussian. (c) represents that the undirected edges connected to $x_{1}$ are oriented since only $x_{1}$ is Gaussian and $\{x_{2}, x_{3}, x_{4}\}$ are non-Gaussian according to Theorem \ref{ancestor finding simple BCA star}. The induced subgraph when $x_{1}$ is removed is shown in the solid part of (d). From (ii) and (iii) in Theorem \ref{ancestor finding BCA star}, $x_{2} \to x_{4}$, $x_{2} \to x_{3}$ are identified and $x_2 \in BCA^{*}_{34}$ is detected as in (f). Removing $x_2$, and we have (g). In (g), both residuals $v_{3}$ and $v_{4}$ in line 48 are Gaussian. Therefore, no further direction is identifiable by Theorem \ref{ancestor finding BCA star}. After applying Algorithm \ref{alg:PC2} to (g), Algorithm \ref{alg: proposed} returns a DEP in (h). 

About Algorithm \ref{alg: proposed}, we have the following theorem. 
\begin{theorem}
    \label{thm:identifiability}
    Algorithm \ref{alg: proposed} generically identifies a true causal DAG up to the distribution-equivalence pattern. 
\end{theorem}

\subsection{Complexity Analysis}
\label{sec: complexity}
In this section, we evaluate the time complexities of the proposed method compared to the PC-LiNGAM. 
Both are the same until a DSEP is obtained using the PC algorithm. 
The main difference is found in the procedure for obtaining a DEP from a DSEP.

The worst case, both for the PC-LiNGAM and for the proposed method, is when a DSEP is an undirected complete graph. In this case, the MEC is the set of any directed complete DAGs. The PC-LiNGAM needs to enumerate all directed complete DAGs. The number of directed complete DAGs with $p$ vertices is $p!$. For each model, the PC-LiNGAM also needs to compute the residuals of the structural equations corresponding to each variable using the OLS. The time complexity of finding the residuals by the OLS is $O(np^3+p^4)$, where $n$ is the sample size. When we use the Shapiro–Wilk test \cite{Shapiro} to test the Gaussianity of residuals, its time complexity is $O(n\log{n})$. Therefore, the time complexity of PC-LiNGAM for an undirected complete graph is $O(p! \cdot (pn \log n + np^3 + p^4))$, which shows that the PC-LiNGAM is infeasible when $p$ is large.

Next, we consider applying the proposed method to a complete graph. Algorithm \ref{alg: proposed} starts from a complete graph $G_{\mathrm{dsep}}$, as shown in Figure \ref{fig: example for proposal}, and while deleting nodes from $G_{\mathrm{dsep}}$ according to the causal order in $G$, it performs a Gaussian test for each variable and checks whether each edge satisfies either (ii) or (iii) of Theorem \ref{ancestor finding BCA star}. 
Since $B_{ij}$ for all $(i,j)$ consists only of a source node in each step, the models (\ref{model: no ancestors BCA star}) are always simple regression models. 
Therefore, the total numbers of Gaussianity tests, regressions in (\ref{model: no ancestors BCA star}), and independence tests are
\begin{align*}          
    p + (p - 1) + \cdots 2 = O(p^2),
\end{align*}
\begin{align*}          
    (p - 1) + (p - 2) + \cdots 2 = O(p^2),
\end{align*}
\begin{align*}          
    &
    2 \cdot \binom{p}{2} + 2 \cdot \binom{p-1}{2} + \cdots + 2 \cdot \binom{2}{2} =O(p^3), 
\end{align*}
respectively. 
Suppose that we use the Hilbert–Schmidt independence criterion (\cite{Gretton2007}, HSIC) as independence tests and the Shapiro–Wilk test \cite{Shapiro} as the Gaussianity tests. 
The time complexity of HSIC is known to be $O(n^{2})$ \cite{Gretton2007} and that of the OLS (\ref{model: no ancestors BCA star}) is $O(n)$. 
In summary, the time complexity of the proposed method is $O(n\log{n}\cdot p^2 + n\cdot p^2 + n^{2}\cdot p^3) = O(n^{2}\cdot p^3)$. 
Therefore, the proposed algorithm is in polynomial time even when $G$ is a complete DAG. 

We also discuss the case where $G$ is a directed tree. 
Then, the DSEP is an undirected tree, and the MEC consists of $p$ directed trees.  
In the PC-LiNGAM, the number of OLS operations is $p \cdot (p-1)$, and the number of Gaussianity tests is $p$. 
The proposed method requires only $p$ Gaussianity tests and $2(p-1)$ independence tests.
Therefore, while that of the PC-LiNGAM is $O(n\log{n} \cdot p^{2})$, the time complexity of the proposed method is $O(n^{2}\cdot p)$. 
In the case of $n \ll p^2$, the proposed method has a lower time complexity than the PC-LiNGAM. 

\subsection{Exception handlings}
\label{sec: exception}
In Section \ref{sec:proposed}, we showed that Algorithm \ref{alg: proposed} can identify a DEP. However, if there are errors in the Gaussianity or independence tests during implementation, $G_{\mathrm{dep}}$ may not be consistent with the DSEP obtained by the PC algorithm.
Consider the examples shown in Figure \ref{fig: both}. 

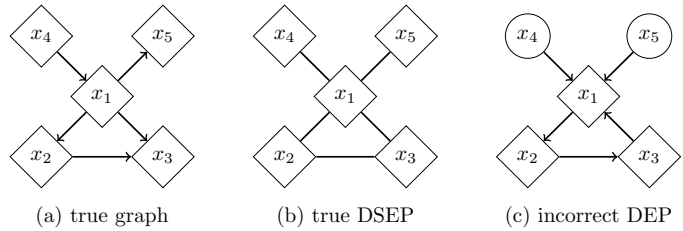
\begin{figure}[b]
    \centering
    \scalebox{0.8}{
    \begin{tikzpicture} 
        \node[draw, diamond] (x_1) at (0, 1) {$x_1$};
        \node[draw, diamond] (x_2) at (-1, 0) {$x_2$};
        \node[draw, diamond] (x_3) at (1, 0) {$x_3$};
        \node[draw, diamond] (x_4) at (-1, 2) {$x_4$};
        \node[draw, diamond] (x_5) at (1, 2) {$x_5$};
        \draw[thick, ->] (x_1) -- (x_2);
        \draw[thick, ->] (x_1) -- (x_3);
        \draw[thick, ->] (x_2) -- (x_3);
        \draw[thick, ->] (x_4) -- (x_1);
        \draw[thick, ->] (x_1) -- (x_5);
        \node[] (X) at (0, -1) {(a) true graph};

        \node[draw, diamond] (x_1) at (4, 1) {$x_1$};
        \node[draw, diamond] (x_2) at (3, 0) {$x_2$};
        \node[draw, diamond] (x_3) at (5, 0) {$x_3$};
        \node[draw, diamond] (x_4) at (3, 2) {$x_4$};
        \node[draw, diamond] (x_5) at (5, 2) {$x_5$};
        \draw[thick, -] (x_1) -- (x_2);
        \draw[thick, -] (x_1) -- (x_3);
        \draw[thick, -] (x_2) -- (x_3);
        \draw[thick, -] (x_1) -- (x_4);
        \draw[thick, -] (x_1) -- (x_5);
        \node[] (X) at (4, -1) {(b) true DSEP};
        
        % \node[draw, diamond] (x_1) at (8, 1) {$x_1$};
        % \node[draw, circle] (x_2) at (7, 0) {$x_2$};
        % \node[draw, circle] (x_3) at (9, 0) {$x_3$};
        % \draw[thick, <-] (x_1) -- (x_2);
        % \draw[thick, <-] (x_1) -- (x_3);
        % \node[] (X) at (8, -1) {\makecell{(c) incorrect graph}};

        \node[draw, diamond] (x_1) at (8, 1) {$x_1$};
        \node[draw, diamond] (x_2) at (7, 0) {$x_2$};
        \node[draw, diamond] (x_3) at (9, 0) {$x_3$};
        \node[draw, circle] (x_4) at (7, 2) {$x_4$};
        \node[draw, circle] (x_5) at (9, 2) {$x_5$};
        \draw[thick, ->] (x_1) -- (x_2);
        \draw[thick, <-] (x_1) -- (x_3);
        \draw[thick, ->] (x_2) -- (x_3);
        \draw[thick, ->] (x_4) -- (x_1);
        \draw[thick, ->] (x_5) -- (x_1);
        \node[] (X) at (8, -1) {\makecell{(c) incorrect DEP}};

    \end{tikzpicture} 
    }
    \caption{
    An example to illustrate how Algorithm \ref{alg: proposed} outputs an incorrect graph with a V-structure detectable by the PC algorithm and a cycle due to errors in Gaussianity tests and independence tests. 
    }
    \label{fig: both}
\end{figure}

In Figure \ref{fig: both}, all disturbances are non-Gaussian, as shown in (a).
(b) shows the corresponding DSEP for (a), in which none of the orientations are identifiable.
We assume that $x_4$ and $x_5$ are incorrectly detected to be Gaussian due to type II errors of the Gaussian test. 
Consequently, the direction $x_5 \to x_1 \leftarrow x_4$ is incorrectly estimated.
Furthermore, if the ancestral relationships among $x_1, x_2$ and $x_3$ are estimated as $x_3 \in Anc_{1}$, $x_1 \in Anc_{2}$, and $x_2 \in Anc_{3}$, the direction $x_1 \leftarrow x_3$ is incorrectly estimated, forming a cycle $x_1 \to x_2 \to x_3 \to x_1$.  
The mixed graph (c) is neither a chain graph nor a DEP.

An algorithm that guarantees the output mixture graph is consistent with a DSEP is preferred.
Algorithm \ref{alg: exception} provides an idea for handling these exceptions. 
The input of Algorithm \ref{alg: exception} is $G_{\mathrm{dep}}$ at line 58 in Algorithm \ref{alg: proposed}. 
Assume that $G_{\mathrm{dep}}=(\bm{X},\tilde{E}_{\mathrm{di}} \cup \tilde{E}_{\mathrm{ud}})$ is inconsistent with
$G_{\mathrm{dsep}}=(\bm{X},E_{\mathrm{di}} \cup E_{\mathrm{di}})$. 
Let $\bm{X}^\prime_{\mathrm{di}}:= \bigcup_{e \in \tilde{E}_{\mathrm{di}} \setminus E_{\mathrm{di}}} e$. 
Then, there exists a weakly connected component $G'$ of the induced subgraph $G_{\mathrm{dep}}(\bm{X}^\prime_{\mathrm{di}})$ that contains a V-structure detectable by the PC algorithm or a cycle. 
In other words, a weakly connected component $G'$ (including $G'$ itself) exists that has either no single source node or two or more source nodes. 
As shown in Section \ref{sec:DMG}, if $G_{\mathrm{dep}}$ is consistent with $G_{\mathrm{dsep}}$, any weakly connected component $G'$ should have only one source node. 
Algorithm \ref{alg: exception} modifies an inconsistent $G'$ into a chain graph such that all the weakly connected component has only one source node. 

\begin{breakablealgorithm}
	\caption{Handling the exceptions} 
	\label{alg: exception} 
	\begin{algorithmic}[1]
		\Require $G_{\mathrm{dep}}$ at line 59 in Algorithm \ref{alg: proposed} and $G_{\mathrm{dsep}}=(\bm{X},E_{\mathrm{di}} \cup E_{\mathrm{ud}})$
		\Ensure A partially DAG 
                $G_{\mathrm{pd}}=(\bm{X},\tilde{E}_{\mathrm{di}} \cup \tilde{E}_{\mathrm{ud}})$ 
                 consistent with $G_{\mathrm{dsep}}$
            \State {$G_{\mathrm{pd}} \leftarrow G_{\mathrm{dep}}(\bm{X},\tilde{E}_{\mathrm{di}} \cup \tilde{E}_{\mathrm{ud}})$}
            % \If {$\exists x_{i} \rightarrow x_{k} \leftarrow x_{j} \in (G_{\mathrm{pd}} \setminus G_{\mathrm{dsep}})$ where $x_{i}$ and $x_j$ are not adjacent}
            \If {
            \hspace{0.2em}
            $\exists$ a cycle in $G_{\mathrm{pd}}$
            \par 
            or 
            \par 
            $\exists x_{i} \rightarrow x_{k} \leftarrow x_{j} \in G_{\mathrm{pd}}$ where $x_{i}$ and $x_j$ are not adjacent, 
            \par 
            and $x_{i} \rightarrow x_{k} \leftarrow x_{j} \notin G_{\mathrm{dsep}}$
            }
%            \State {$\bm{X}_{\mathrm{di}} \leftarrow \{x_{i} \mid (\{x_{i} \to x_{j}\} \in \tilde{E}_{\mathrm{di}} \setminus E_{\mathrm{di}}) \vee (\{x_{j} \to x_{i}\} \in \tilde{E}_{\mathrm{di}} \setminus E_{\mathrm{di}})\}$ }
            \State{$\bm{X}_{\mathrm{di}} \leftarrow \bigcup_{e \in \tilde{E}_\mathrm{di} \setminus E_\mathrm{di}} e$}
            \State {$G_{\mathrm{di}} \leftarrow (\bm{X}_{\mathrm{di}}, \tilde{E}_{\mathrm{di}} \setminus E_{\mathrm{di}})$}
            \State {Find all weak connected components $\mathcal{C}$ of $G_{\mathrm{di}}$}
            \ForAll{$G^\prime = (\bm{X}^\prime_{\mathrm{di}}, E^{\prime}_{\mathrm{di}}) \in \mathcal{C}$}
                \If {$\exists$ loop in $G_{\mathrm{di}}$
                \par \hspace{1.8em}
                or 
                \par \hspace{1.8em}
                $\exists x_{i} \rightarrow x_{k} \leftarrow x_{j} \in G_{\mathrm{di}}$ where $x_{i}$ and $x_j$ are not adjacent, 
                \par \hspace{1.8em}
                and $x_{i} \rightarrow x_{k} \leftarrow x_{j} \notin G_{\mathrm{dsep}}$}
                    \State{Find all source nodes in $G^{\prime}$ and randomly select one as $x_{0}$}
                    \If{$x_{0} = \mathrm{NULL}$}
                        \State{Randomly select one nodes in $\bm{X}^{\prime}_{\mathrm{di}}$ as $x_{0}$}
                    \EndIf
                    \State{$X_{\mathrm{closed}} \leftarrow \emptyset$}
                    \State{$X_{\mathrm{open}} \leftarrow \{x_{0}\}$}
                    \While{$|X_{\mathrm{closed}}| \ne |\bm{X}^\prime_{\mathrm{di}}|$}
                    \State{$X_{\mathrm{open}} \leftarrow X_{\mathrm{open}} \setminus x_{0}$}
                    \State{Find the set $Adj_0$ of all adjacent nodes of $x_{0}$ in $G^{\prime}$}
                    \State{$X_{\mathrm{open}} \leftarrow X_{\mathrm{open}}\cup (Adj_0\setminus X_{\mathrm{closed}})$}
                    \ForAll{$x_{i} \in Adj_0 \setminus X_{\mathrm{closed}}$}
                        \If{$\{x_0 \leftarrow x_{i}\} \in \tilde{E}_{\mathrm{di}}$}
                            \State {$\tilde{E}_{\mathrm{di}} \leftarrow \tilde{E}_{\mathrm{di}} \cup 
                        \{x_0 \to x_{i}\} \setminus \{x_i \to x_{0}\}$}
                        \EndIf
                    \EndFor
                    \State{$X_{\mathrm{closed}} \leftarrow X_{\mathrm{closed}} \cup x_{0}$}
                    % \State{Randomly select a $x_{i}\in X_{\mathrm{open}} \cap Adj_0\setminus X_{\mathrm{closed}}$ as new $x_{0}$}
                    \State{Randomly select a $x_{i}\in X_{\mathrm{open}} \cap Adj_0\setminus X_{\mathrm{closed}}$ as new $x_{0}$}
                    \If {$x_{0} = \mathrm{NULL}$}
                    \State{Randomly select a $x_{i}\in X_{\mathrm{open}}$ as new $x_{0}$}
                    \EndIf
                    \EndWhile
                \EndIf
            \EndFor
            \EndIf
            \State Apply the orientation rule to $G_{\mathrm{pd}}$
            \\ \Return $G_{\mathrm{pd}}$
    \end{algorithmic}
\end{breakablealgorithm}

In lines 3 and 4, the directed subgraph $G_{\mathrm{di}}$ containing a V-structure or a cycle is extracted.
$G_{\mathrm{di}}$ may be disconnected, and line 5 finds the set of the weakly connected components $\mathcal{C}$ of $G_{\mathrm{di}}$. 
If $G' \in \mathcal{C}$ contains a V-structure or a cycle and has source nodes, randomly select one of them and set it as $x_0$ (line 8). 
If $G' \in \mathcal{C}$ contains a cycle and has no source nodes, randomly select one of the variables in $G'$ and set it as $x_0$ (line 9-11). 
In lines 24-28, some edges of $G'$ are reversed %to modify $G'$ into a chain graph with $x_{0}$ as the only source node. 
so that $G'$ is consistent with $G_{\mathrm{dsep}}$. 
The resulting chain graph neither contains a V-structure that is detectable by the PC algorithm nor a cycle.

Algorithm \ref{alg: exception} is based on the breadth-first search. 

\begin{figure}[htbp]
    \centering
    \scalebox{0.8}{
    \begin{tikzpicture} 
        \node[draw, diamond] (x_1) at (0, 1) {$x_1$};
        \node[draw, diamond] (x_2) at (-1, 0) {$x_2$};
        \node[draw, diamond] (x_3) at (1, 0) {$x_3$};
        \node[draw, circle] (x_4) at (-1, 2) {$x_4$};
        \node[draw, circle] (x_5) at (1, 2) {$x_5$};
        \draw[thick, ->] (x_1) -- (x_2);
        \draw[thick, <-] (x_1) -- (x_3);
        \draw[thick, ->] (x_2) -- (x_3);
        \draw[thick, ->] (x_4) -- (x_1);
        \draw[thick, ->] (x_5) -- (x_1);
        \node[] (X) at (0, -1) {(a) incorrect DEP};

        \node[] (s) at (1.5, 1) {};
        \node[] (e) at (2.5, 1) {};
        \draw[thick, ->] (s) -- (e);

        \node[] (x_0) at (3, 3) {$x_0$};
        \node[draw, diamond] (x_1) at (4, 1) {$x_1$};
        \node[draw, diamond] (x_2) at (3, 0) {$x_2$};
        \node[draw, diamond] (x_3) at (5, 0) {$x_3$};
        \node[draw, circle] (x_4) at (3, 2) {$x_4$};
        \node[draw, circle] (x_5) at (5, 2) {$x_5$};
        
        \draw[thick, ->] (x_0) -- (x_4);
        
        \draw[thick, ->] (x_1) -- (x_2);
        \draw[thick, <-] (x_1) -- (x_3);
        \draw[thick, ->] (x_2) -- (x_3);
        \draw[thick, <-] (x_1) -- (x_4);
        \draw[thick, <-] (x_1) -- (x_5);
        \node[] (X) at (4, -1) {(b)};
        
        \node[] (s) at (5.5, 1) {};
        \node[] (e) at (6.5, 1) {};
        \draw[thick, ->] (s) -- (e);

        \node[] (x_0) at (8, 2) {$x_0$};
        \node[draw, diamond] (x_1) at (8, 1) {$x_1$};
        \node[draw, diamond] (x_2) at (7, 0) {$x_2$};
        \node[draw, diamond] (x_3) at (9, 0) {$x_3$};
        \node[draw, circle] (x_4) at (7, 2) {$x_4$};
        \node[draw, circle] (x_5) at (9, 2) {$x_5$};

        \draw[thick, ->] (x_0) -- (x_1);
        
        \draw[thick, ->] (x_1) -- (x_2);
        \draw[thick, ->] (x_1) -- (x_3);
        \draw[thick, ->] (x_2) -- (x_3);
        \draw[thick, ->] (x_4) -- (x_1);
        \draw[thick, <-] (x_5) -- (x_1);
        \node[] (X) at (8, -1) {\makecell{(c)}};

        \node[] (s) at (8, -1.5) {};
        \node[] (e) at (8, -2.5) {};
        \draw[thick, ->] (s) -- (e);

        \node[] (x_0) at (9, -3) {$x_0$};
        \node[draw, diamond] (x_1) at (8, -5) {$x_1$};
        \node[draw, diamond] (x_2) at (7, -6) {$x_2$};
        \node[draw, diamond] (x_3) at (9, -6) {$x_3$};
        \node[draw, circle] (x_4) at (7, -4) {$x_4$};
        \node[draw, circle] (x_5) at (9, -4) {$x_5$};

        \draw[thick, ->] (x_0) -- (x_5);
        
        \draw[thick, ->] (x_1) -- (x_2);
        \draw[thick, ->] (x_1) -- (x_3);
        \draw[thick, ->] (x_2) -- (x_3);
        \draw[thick, ->] (x_4) -- (x_1);
        \draw[thick, <-] (x_5) -- (x_1);
        \node[] (X) at (8, -7) {\makecell{(d)}};

        \node[] (s) at (6.5, -5) {};
        \node[] (e) at (5.5, -5) {};
        \draw[thick, ->] (s) -- (e);

        \node[] (x_0) at (3, -5) {$x_0$};
        \node[draw, diamond] (x_1) at (4, -5) {$x_1$};
        \node[draw, diamond] (x_2) at (3, -6) {$x_2$};
        \node[draw, diamond] (x_3) at (5, -6) {$x_3$};
        \node[draw, circle] (x_4) at (3, -4) {$x_4$};
        \node[draw, circle] (x_5) at (5, -4) {$x_5$};

        \draw[thick, ->] (x_0) -- (x_2);
        
        \draw[thick, ->] (x_1) -- (x_2);
        \draw[thick, ->] (x_1) -- (x_3);
        \draw[thick, ->] (x_2) -- (x_3);
        \draw[thick, ->] (x_4) -- (x_1);
        \draw[thick, <-] (x_5) -- (x_1);
        \node[] (X) at (4, -7) {\makecell{(e)}};

        \node[] (s) at (2.5, -5) {};
        \node[] (e) at (1.5, -5) {};
        \draw[thick, ->] (s) -- (e);

        \node[] (x_0) at (1, -5) {$x_0$};
        \node[draw, diamond] (x_1) at (0, -5) {$x_1$};
        \node[draw, diamond] (x_2) at (-1, -6) {$x_2$};
        \node[draw, diamond] (x_3) at (1, -6) {$x_3$};
        \node[draw, circle] (x_4) at (-1, -4) {$x_4$};
        \node[draw, circle] (x_5) at (1, -4) {$x_5$};

        \draw[thick, ->] (x_0) -- (x_3);
        
        \draw[thick, ->] (x_1) -- (x_2);
        \draw[thick, ->] (x_1) -- (x_3);
        \draw[thick, ->] (x_2) -- (x_3);
        \draw[thick, ->] (x_4) -- (x_1);
        \draw[thick, <-] (x_5) -- (x_1);
        \node[] (X) at (0, -7) {\makecell{(f) modified DEP}};
        
    \end{tikzpicture} 
    }
    \caption{
    An example to illustrate the flow of Algorithm \ref{alg: exception} to handle the exceptions in Figure \ref{fig: both}.
    }
    \label{fig: handle both exceptions}
\end{figure}
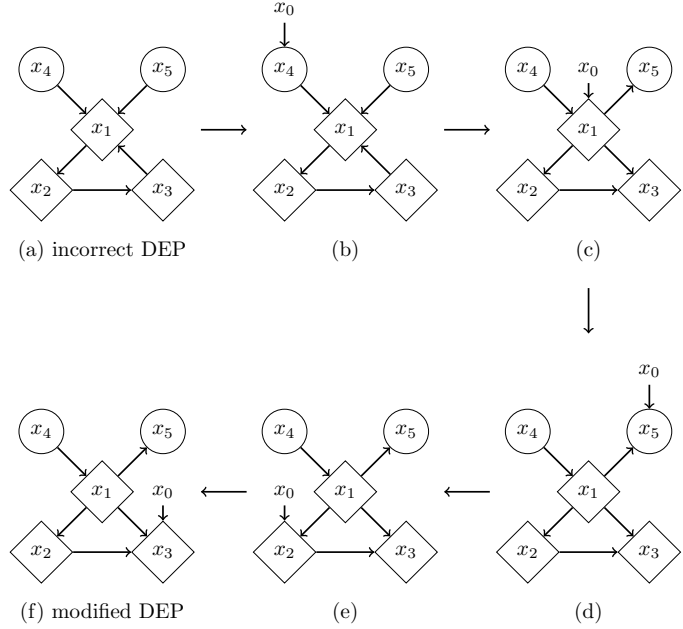

Figure \ref{fig: handle both exceptions} shows the flow of Algorithm \ref{alg: exception} to handle the exceptions depicted in Figure \ref{fig: both}.
Figure \ref{fig: handle both exceptions} (a) depicts the same incorrect DEP as in Figure \ref{fig: both} (c). 
Let $x_4$ be an initial $x_0$ as shown in (b), and 
%Assume that the initially selected source node $x_{0}$ is $x_4$ and 
initialize $Adj_0 = \{x_{1}\}$ and $X_{\mathrm{open}} = \{x_{1}\}$. 
Since $x_{4} \to x_{1}$ is the current direction, there is no need to change the direction of the edge, and $x_{4}$ is added to $X_{\mathrm{closed}}$.
Next, select $x_1$ from $X_{\mathrm{open}} \cap Adj_0 \setminus X_{\mathrm{closed}}$ as the new $x_{0}$. 
%When $x_{1}$ becomes the new $x_{0}$, 
Then update $Adj_0$ and $X_{\mathrm{open}}$ to $\{x_{4}, x_{5}, x_{2}, x_{3}\}$ and $\{x_{5}, x_{2}, x_{3}\}$, respectively. 
Since $Adj_0 \setminus X_{\mathrm{closed}} = \{x_{5}, x_{2}, x_{3}\}$, invert the directions of $x_{5} \to x_{1}$ and $x_{3} \to x_{1}$ as shown in (c).
After adding $\{x_1\}$ to $X_{\mathrm{closed}}$, select $x_{5}$ as the new $x_{0}$ as shown in (d).
Since $Adj_0$ of $x_{5}$ is empty, no directions are modified at this step, and 
$X_{\mathrm{open}}$ is updated to $\{2,3\}$. 
If $x_2$ and $x_3$ are subsequently selected as $x_0$, no further changes occur in edge orientation.
With all vertices traversed, the resulting output is the chain graph in (f).

Since the number of edges in $G_{\mathrm{di}}$ is at most $p(p-1)/2$, the number of operations required to reverse the directions of edges in $G_{\mathrm{pd}}$ is also at most $p(p-1)/2$. Furthermore, the maximum number of operations required to find the source node for all weakly connected components $\mathcal{C}$ is also at most $O(p)$. Therefore, Algorithm \ref{alg: exception} is also a polynomial time.

Algorithm \ref{alg: exception} randomly generates a chain graph that is consistent with the DSEP. 
However, the output chain graph may differ depending on the order of selecting $x_{0}$.
Also, the output of Algorithm \ref{alg: exception} may not be a chain graph that minimizes the changes to the edge orientations of an incorrect DEP. 
Improving the handling of V-structures and cycles in an incorrect DEP remains a topic for future work. 

\section{Numerical Experiments}
\label{sec: experiment}
We performed numerical experiments to confirm the computational efficiency of the proposed method compared to the PC-LiNGAM. In this section, we describe the details of the numerical experiments and present the results of the experiments. Since the proposed method and the PC-LiNGAM have the same procedure for obtaining a DSEP using the PC algorithm, we compare the computation time of the procedure for obtaining a DEP from a DSEP. 
We compare the CPU time of the two methods in the worst case of time complexity where true $G$ is a directed complete DAG, i.e., the DSEP is an undirected complete graph. 
Section \ref{sec: experiment settings} describes the details of the experimental settings. Section \ref{sec: result and discussion} presents the experimental results and discusses the results.

\subsection{Experimental Settings}
\label{sec: experiment settings}
This subsection summarizes the experimental settings. 
The number of variables $p$ in a DAG was set to $\{5, 6, 7\}$.
The sample size $n$ was set to $\{$1500, 2000, 3000, 5000, 10000$\}$. 
Gaussian and non-Gaussian disturbances were generated from $\mathrm{N}(0,1)$ and lognormal distributions $\mathrm{Lognormal}(0,1)$ with expectation standardized to 0, respectively. 
The number of non-Gaussian disturbances was randomly set to more than $\lfloor p/3 \rfloor$ and less than $p$ for each iteration. 
The nonzero elements of the coefficient matrix $B$ were randomly generated from the uniform distribution $U(0.5, 1)$ to satisfy the faithfulness assumption with probability one.
The number of iterations for a fixed $(p, n)$ was set to 50. 

We used the Hilbert–Schmidt independence criterion (HSIC) \cite{Gretton2007} for independence tests in the proposed method. The Shapiro–Wilk test \cite{Shapiro} was used for the Gaussianity tests in both the PC-LiNGAM and the proposed method.
%The time complexity of HSIC is $O(n^2)$, and that of the Shapiro-Wilk test is $O(n\log{n})$.
The significance levels for the HSIC and the Shapiro-Wilk test were set to $0.001$ and $0.05$, respectively. 
In the experiments with the proposed method, HSIC was performed with a random sample of size 1500 out of $n$ for each fixed $(p, n)$ to reduce the computation time. 

In this experiment, Algorithm \ref{alg: exception} was not applied because the sample size was set to be large, and a cycle discussed in Section \ref{sec: exception} is expected to occur only with low probability.

To evaluate the performance of the PC-LiNGAM and the proposed method, for each experimental group $(p, n)$, we recorded the total number of incorrectly estimated DEPs in 50 iterations and the CPU time (in seconds) required for estimating 50 DEPs.

All experiments were conducted on the same workstation equipped with a 3.3GHz Core i9 processor and 128 GB memory.

\subsection{Results and Discussion}
\label{sec: result and discussion}
In this subsection, we present and discuss the experimental results. 

Figures \ref{fig: results0}, \ref{fig: results} (a), (c), and (e) illustrate the CPU time for estimating 50 DEPs using the proposed methods and the PC-LiNGAM with $p=5,6,7$, respectively. 
Figures \ref{fig: results} (b), (d), and (f) show the number of DEPs that were incorrectly estimated by the proposed method and the PC-LiNGAM for $p=5,6,7$, respectively, in the 50 iterations.

From these figures, we observe that both methods do not differ significantly in estimation accuracy, but the proposed method has a far faster computation time when $p=7$. 
Figure \ref{fig: results0} shows that as $p$ increases, the CPU time for the PC-LiNGAM increases rapidly, while the CPU time for the proposed method increases slowly. 
This result is also consistent with the results in Section \ref{sec: complexity}, where the PC-LiNGAM is factorial time, and the proposed method is polynomial time.
The CPU time of the proposed method at $p=7$ is less than $1/10$ of that of the PC-LiNGAM. 
If $p$ exceeds 10, the PC-LiNGAM will not be able to output an estimate of a DEP in a practical amount of time.
When $p=5$ and when $(p,n)=(6,1500), (6,2000)$, the PC-LiNGAM has faster CPU time, but the proposed method is faster in CPU time when $n \ge 3000$, even with $p=6$.

As mentioned in the previous subsection, in this experiment, HSIC was performed using a random sample of size 1500, even when the sample size was larger than 1500.
Since the sample size used for HSIC is fixed, the rate of increase in CPU time for the proposed method against the sample size is moderate. Moreover, the estimation accuracy is not significantly different from that of the PC-LiNGAM.
If the sample size for HSIC is fixed when $G$ is a tree, the proposed method's time complexity is reduced to $O(n\log n \cdot p)$ using the result in Section \ref{sec: complexity}, which is superior to that of the PC-LiNGAM.
Fixing the sample size for HSIC to an appropriate size may reduce the CPU time of the proposed algorithm even when $G$ is sparse and the sample size is large.

In summary, these experiments confirmed that the proposed method can estimate DEPs with reasonably high accuracy and requires far less computation time compared to the PC-LiNGAM.

\begin{figure}[htp]
    \centering
    % 第一行子图
    \scalebox{0.2}{
    \includegraphics{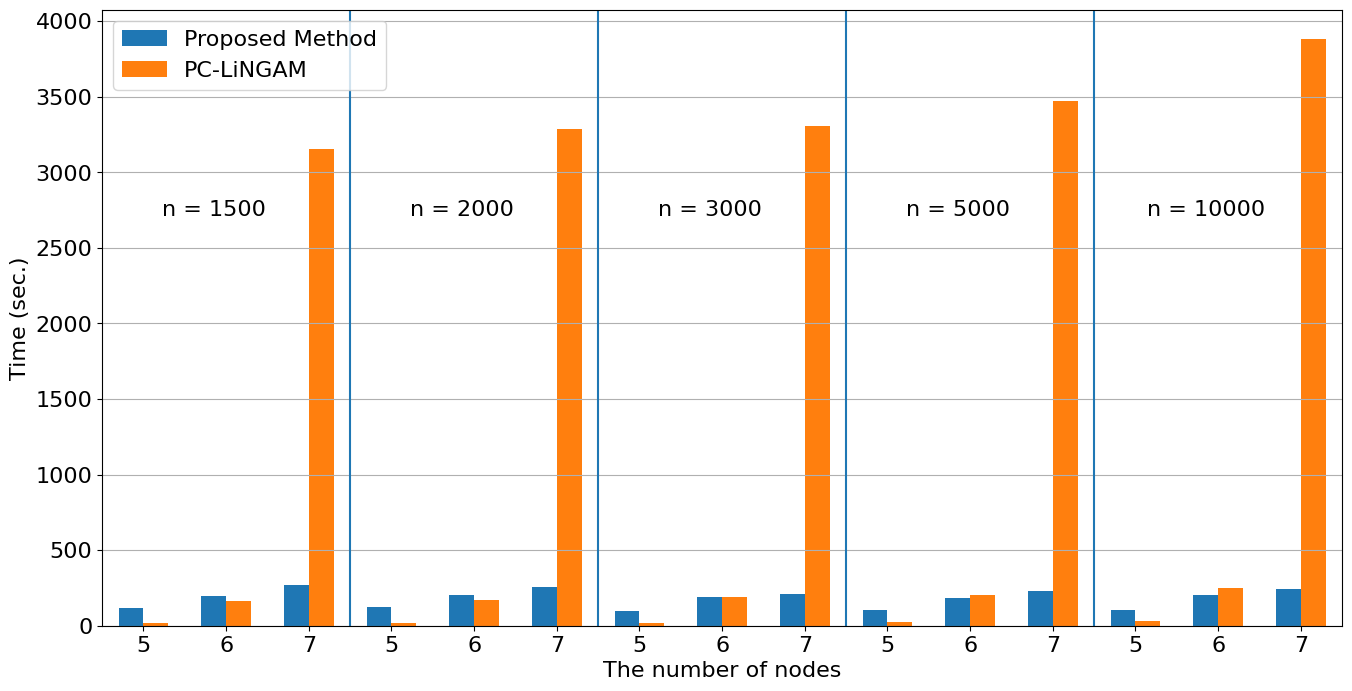}
    }
    \caption{CPU time of the proposed algorithm and the PC-LiNGAM against the dimension of variables.}
    \label{fig: results0}
\end{figure}
\begin{figure}[htp]
    \centering
    % 第一行子图
    \includegraphics[width=\linewidth]{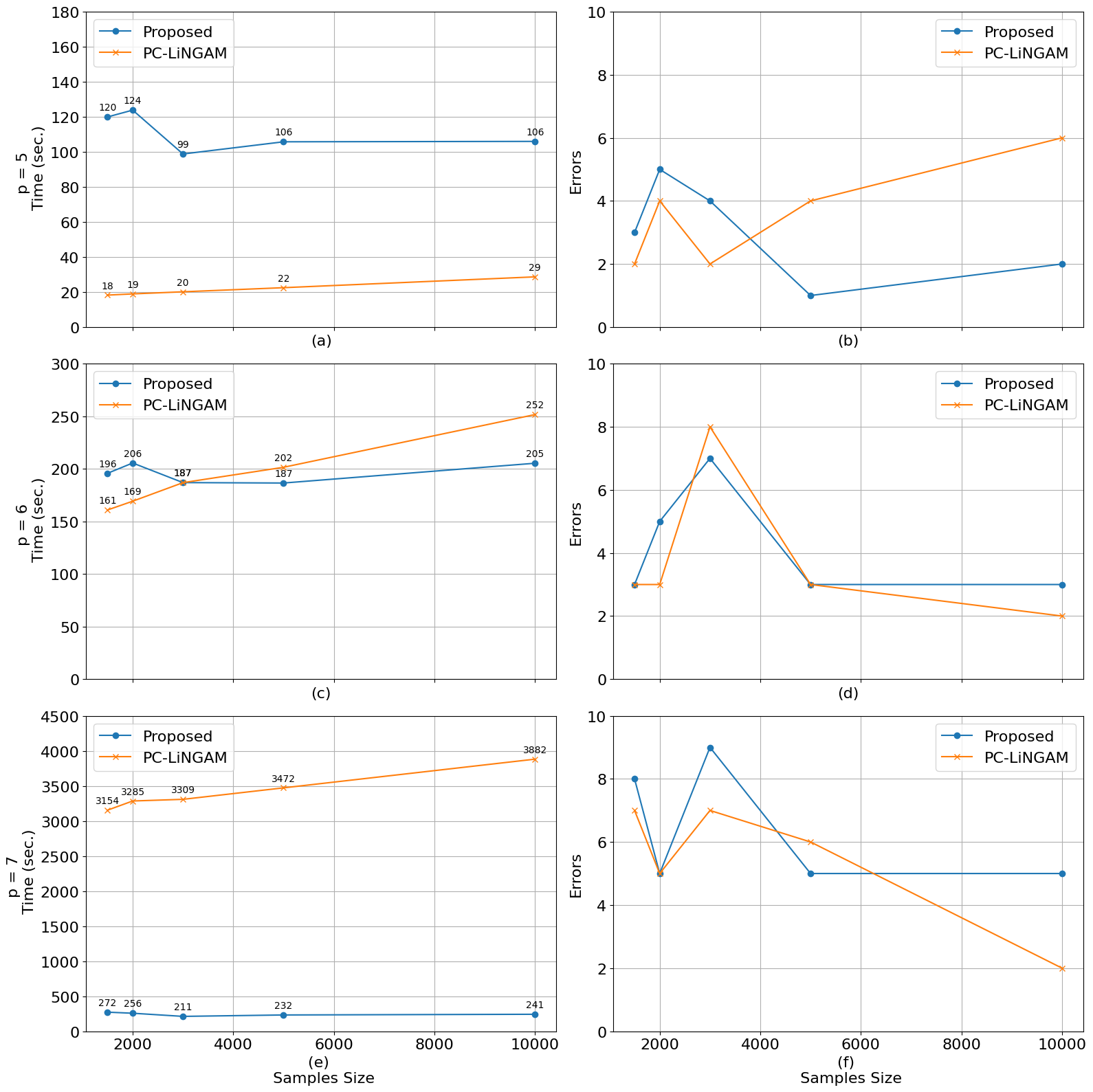}
    \caption{CPU time and estimation accuracy of the proposed algorithm and the PC-LiNGAM for $p=5,6,7$: 
    The figures in the left column show the CPU times in seconds against the sample size. 
    The figures in the right column show the number of incorrectly estimated DEPs in the 50 iterations against the sample size.}
    \label{fig: results}
\end{figure}

\section{Conclusion}
This paper proposes a new algorithm for learning distribution-equivalence patterns of a causal graph in linear causal models. 
We generalized the ancestor-finding proposed by Maeda and Shimizu \cite{Maeda2020} to the case where Gaussian disturbances are included in the linear causal models. We used it to determine the orientation of the undirected edges of the d-separation-equivalence pattern (DSEP).
We showed that the proposed method runs in polynomial time, whereas the PC-LiNGAM runs in factorial time in the worst case.

We assumed that a DSEP estimated by the PC algorithm is correct and then performed numerical experiments to estimate a distribution-equivalence pattern (DEP) from a DSEP in the case where the true causal DAG is a directed complete DAG. 
The results showed that the proposed method and the PC-LiNGAM do not differ significantly in the estimation accuracy, but the proposed method dramatically reduces the computation time. 
When the number of variables is 7, the proposed method exhibited far faster CPU time compared to the PC-LiNGAM. 

We did not perform any experiments implementing the proposed method, including estimating a DSEP using the PC algorithm. 
Since the PC algorithm is exponential in computation time, the entire algorithm of the proposed method is also exponential time. 

When the true causal DAG is sparse, divide-and-conquer approaches (e.g., \cite{Cai2013, Zhang2020, cai2024learning}) might accelerate the PC algorithm. 
The combination of the divide-and-conquer PC algorithm and the proposed method may make it possible to compute DEPs for high-dimensional and sparse causal DAGs in a practical amount of time. 

The problem with Algorithm \ref{alg: proposed} is that the output $G_\mathrm{dep}$ may contain V-structures or directed cycles that are inconsistent with a DSEP. In Section \ref{sec: exception}, we provided Algorithm \ref{alg: exception}, which outputs a chain graph that is consistent with a DSEP by removing inconsistent V-structures and cycles.
Algorithm \ref{alg: exception} randomly generates a chain graph that is consistent with a DSEP. If the choice of $x_0$ changes, the output chain graph may also change, and the plausibility of the output chain graph is not fully evaluated. Especially when the sample size is small compared to the dimension of the variables, the impact of such exception handling on estimation accuracy is expected to be significant. A better exception handling is left as a future task. 

This paper does not assume the existence of latent confounders. As mentioned in Section \ref{sec:pc}, FCI (\cite{Spirtes}) is a generalization of the PC algorithm to cases with the presence of latent confounders. 
RCD is an algorithm for identifying causal DAGs for the model that generalizes LiNGAM to account for the presence of latent confounders. 
Similar to the proposed method, one possible direction is to combine the FCI and RCD to identify the causal graphs that define linear causal models that allow for the presence of Gaussian disturbances and latent confounders. 
This would also be a topic for future research. 
%We also hope that the method proposed in this paper can be extended to handle this case efficiently in future research.

\bibliographystyle{plain}
\bibliography{main}

\begin{thebibliography}{10}

\bibitem{Andersson1997}
Steen~A. Andersson, David Madigan, and Michael~D. Perlman.
\newblock A characterization of {M}arkov equivalence classes for acyclic digraphs.
\newblock {\em The Annals of Statistics}, 25(2):505--541, 1997.

\bibitem{cai2024learning}
Ming Cai and Hisayuki Hara.
\newblock Learning causal graphs using variable grouping according to ancestral relationship.
\newblock {\em arXiv preprint arXiv:2403.14125}, 2024.

\bibitem{Cai2013}
Ruichu Cai, Zhenjie Zhang, and Zhifeng Hao.
\newblock {SADA}: A general framework to support robust causation discovery.
\newblock In {\em International Conference on Machine Learning}, pages 208--216. PMLR, 2013.

\bibitem{Chickering2002}
David~Maxwell Chickering.
\newblock Optimal structure identification with greedy search.
\newblock {\em Journal of Machine Learning Research}, 3:507--554, 2002.

\bibitem{Cramér}
Harald Cram{\'e}r.
\newblock {\em Random variables and probability distributions}.
\newblock 36. Cambridge University Press, 2004.

\bibitem{Darmois1953}
George Darmois.
\newblock Analyse g{\'e}n{\'e}rale des liaisons stochastiques: etude particuli{\`e}re de l'analyse factorielle lin{\'e}aire.
\newblock {\em Revue de l'Institut international de statistique}, pages 2--8, 1953.

\bibitem{dfs2011}
Mathias Drton, Rina Foygel, and Seth Sullivant.
\newblock Global identifiability of linear structural equation models.
\newblock {\em The Annals of Statistics}, 39:865--886, 2011.

\bibitem{Giudice2023}
Enrico Giudice, Jack Kuipers, and Giusi Moffa.
\newblock The dual {PC} algorithm and the role of {G}aussianity for structure learning of {B}ayesian networks.
\newblock {\em International Journal of Approximate Reasoning}, 161:108975, 2023.

\bibitem{Gretton2007}
Arthur Gretton, Kenji Fukumizu, Choon Teo, Le~Song, Bernhard Sch{\"o}lkopf, and Alex Smola.
\newblock A kernel statistical test of independence.
\newblock {\em Advances in Neural Information Processing Systems}, 20, 2007.

\bibitem{hoyer2008b}
Patrik~O. Hoyer, Aapo Hyvarinen, Richard Scheines, Peter Spirtes, Joseph Ramsey, Gustavo Lacerda, and Shohei Shimizu.
\newblock Causal discovery of linear acyclic models with arbitrary distributions.
\newblock In {\em Proceedings of the Twenty-Fourth Conference on Uncertainty in Artificial Intelligence, UAI2008}, pages 282--289, 2008.

\bibitem{Hoyer2008}
Patrik~O. Hoyer, Dominik Janzing, Joris~M Mooij, Jonas Peters, and Bernhard Sch{\"o}lkopf.
\newblock Nonlinear causal discovery with additive noise models.
\newblock {\em Advances in Neural Information Processing Systems}, 21, 2008.

\bibitem{hoyer2008estimation}
Patrik~O. Hoyer, Shohei Shimizu, Antti~J. Kerminen, and Markus Palviainen.
\newblock Estimation of causal effects using linear non-{G}aussian causal models with hidden variables.
\newblock {\em International Journal of Approximate Reasoning}, 49(2):362--378, 2008.

\bibitem{Hyvarinen2002}
Aapo Hyv{\"a}rinen, Juha Karhunen, and Erkki Oja.
\newblock Independent component analysis.
\newblock {\em Studies in Informatics and Control}, 11(2):205--207, 2002.

\bibitem{hyvarinen10a}
Aapo Hyv{\"a}rinen, Kun Zhang, Shohei Shimizu, and Patrik~O. Hoyer.
\newblock Estimation of a structural vector autoregression model using non-{G}aussianity.
\newblock {\em Journal of Machine Learning Research}, 11(56):1709--1731, 2010.

\bibitem{kb2007}
Markus Kalisch and Peter B{\"u}hlman.
\newblock Estimating high-dimensional directed acyclic graphs with the {PC}-algorithm.
\newblock {\em Journal of Machine Learning Research}, 8(3):613--636, 2007.

\bibitem{fastPC2016}
Thuc~Duy Le, Tao Hoang, Jiuyong Li, Lin Liu, Huawen Liu, and Shu Hu.
\newblock A fast {PC} algorithm for high dimensional causal discovery with multi-core {PC}s.
\newblock {\em IEEE/ACM Transactions on Computational Biology and Bioinformatics}, 16(5):1483--1495, 2016.

\bibitem{Maeda2020}
Takashi~Nicholas Maeda and Shohei Shimizu.
\newblock {RCD}: Repetitive causal discovery of linear non-{G}aussian acyclic models with latent confounders.
\newblock In {\em International Conference on Artificial Intelligence and Statistics}, pages 735--745. PMLR, 2020.

\bibitem{meek1995}
Christopher Meek.
\newblock Strong completeness and faithfulness in {B}ayesian networks.
\newblock In {\em Proceedings of the 11th Conference on Uncertainty in Artificial Intelligence, UAI'95}, pages 411--418, 1995.

\bibitem{Shapiro}
Samuel~Sanford Shapiro and Martin~Bradbury Wilk.
\newblock An analysis of variance test for normality (complete samples).
\newblock {\em Biometrika}, 52(3--4):591--611, 1965.

\bibitem{Shimizu2006}
Shohei Shimizu, Patrik~O. Hoyer, Aapo Hyv{\"a}rinen, Antti Kerminen, and Michael Jordan.
\newblock A linear non-{G}aussian acyclic model for causal discovery.
\newblock {\em Journal of Machine Learning Research}, 7:2003--2030, 2006.

\bibitem{Shimizu2011}
Shohei Shimizu, Takanori Inazumi, Yasuhiro Sogawa, Aapo Hyv{\"a}rinen, Yoshinobu Kawahara, Takashi Washio, Patrik~O. Hoyer, and Kenneth Bollen.
\newblock Directlingam: A direct method for learning a linear non-{G}aussian structural equation model.
\newblock {\em Journal of Machine Learning Research}, 12:1225--1248, 2011.

\bibitem{Skitovich1953}
Viktor~Pavlovich Skitovich.
\newblock On a property of the normal distribution.
\newblock {\em Doklady Akademii Nauk}, 89:217--219, 1953.

\bibitem{sg1991}
Peter Spirtes and Clark Glymour.
\newblock An algorithm for fast recovery of sparse causal graphs.
\newblock {\em Social Science Computer Review}, 9:62--72, 1991.

\bibitem{Spirtes}
Peter Spirtes, Clark Glymour, and Richard Scheines.
\newblock {\em Causation, prediction, and search}.
\newblock MIT press, 2001.

\bibitem{tashiro2014parcelingam}
Tatsuya Tashiro, Shohei Shimizu, Aapo Hyv{\"a}rinen, and Takashi Washio.
\newblock Parcelingam: A causal ordering method robust against latent confounders.
\newblock {\em Neural Computation}, 26(1):57--83, 2014.

\bibitem{vp1992}
Thomas Verma and Judea Pearl.
\newblock An algorithm for deciding if a set of observed independence has a causal explanation.
\newblock In {\em Proceedings of the 8th Conference on Uncertainty in Artificial Intelligence, UAI'92}, pages 323--330. Elsevier, 1992.

\bibitem{Zhang2020}
Hao Zhang, Shuigeng Zhou, Chuanxu Yan, Jihong Guan, Xin Wang, Ji~Zhang, and Jun Huan.
\newblock Learning causal structures based on divide and conquer.
\newblock {\em IEEE Transactions on Cybernetics}, 52(5):3232--3243, 2020.

\bibitem{Zhang2012}
Kun Zhang and Aapo Hyv{\"a}rinen.
\newblock On the identifiability of the post-nonlinear causal model.
\newblock {\em arXiv preprint arXiv:1205.2599}, 2012.

\end{thebibliography}

\appendix
\section{Appendix}
\subsection{Some basic facts on the linear causal model}
This section summarizes some basic facts necessary for the proof of Theorem \ref{ancestor finding simple} and \ref{ancestor finding}.
Consider the linear acyclic model (\ref{eq:SEM}). 
Denote by $b_{j,i}$ the $(j,i)$-element of $B$. 
The model (\ref{eq:SEM}) is rewritten by
\[
\bm{X} = (I-B)^{-1} \bm{\epsilon},
\]
and $(I-B)^{-1}$ is also transformed into a lower triangular matrix with all diagonal elements equal to one by permuting the rows and the columns. 
Let $d_{j,i}$ be the $(j,i)$-element of $(I-B)^{-1}$. 
The following lemma is well known. 
\begin{lemma}[e.g. \cite{dfs2011}]
\label{lemma: relation}
Let $\mathcal{P}(i, j)$ denote the set of directed paths from $x_i$ to $x_j$ in $G$.
Then, $d_{j,i}$ %, $j > i$ 
is written by
\begin{align*}
    d_{j, i} &= 
    \sum_{\pi \in \mathcal{P}(i,j)} \prod_{x_k \to x_l \in \pi} b_{l,k}, 
    \label{equation: d and b}
\end{align*}
which is the total effect from $x_i$ to $x_j$. 
\end{lemma}
We note that $x_i$ is expressed as
\[
x_i = \sum_{k: x_k \in Anc_i} d_{i,k}\epsilon_k.
\]
From the faithfulness assumption, $d_{i,k} \ne 0$ if $x_k \in Anc_i$. 
\begin{lemma}
    \label{lem:noBCA}
    Assume that $x_i \in Anc_j$ and that $BCA_{ij} = \emptyset$. 
    Then, the following two conditions hold.
    \begin{compactenumi}
        \item For $x_k \in Anc_i$, $d_{j,k} = d_{j,i} d_{i,k}$.
        \item $d_{j,i}$ is expressed as
        \[
        d_{j,i} = \frac{\mathrm{Cov}(x_i,x_j)}{\mathrm{Var}(x_i)}. 
        \]        
    \end{compactenumi}
\end{lemma}
\begin{proof}\ 
\begin{compactenumi}
    \item Since $BCA_{ij} = \emptyset$, all path in ${\cal P}(k,j)$ include $x_i$. 
    Therefore
    \begin{align*}
        d_{j,k} &= \sum_{\pi \in \mathcal{P}(k,j)} \prod_{x_h \to x_{l} \in \pi} b_{l,h}\\
        &= \sum_{\substack{\pi \in \mathcal{P}(k,i)\\
        \pi' \in \mathcal{P}(i,j)}} \prod_{x_h \to x_{l} \in \pi} b_{l,h} \cdot
        \prod_{x_{h'} \to x_{l'} \in \pi'} b_{l',h'}\\
        &= \sum_{\pi \in \mathcal{P}(k,i)} \prod_{x_h \to x_{l} \in \pi} b_{l,h} \cdot
        \sum_{\pi' \in \mathcal{P}(i,j)} \prod_{x_{h'} \to x_{l'} \in \pi'} b_{l',h'} 
        = d_{i,k} d_{j,i}
    \end{align*}
    \item Since $x_i$ and $x_j$ is expressed as
        \[
        x_i = \sum_{k : x_k \in Anc_i \cup \{x_i\}} d_{i,k} \epsilon_k, \quad
        x_j = \sum_{l : x_l \in Anc_j \cup \{x_j\}} d_{j,l} \epsilon_l, 
        \]
        $\mathrm{Cov}(x_i,x_j)$ and $\mathrm{Var}(x_i)$ is written by 
        \begin{align*}
            \mathrm{Cov}(x_i,x_j) &= 
            \sum_{k : x_k \in Anc_i \cup \{x_i\}} d_{i,k} d_{j,k} \mathrm{Var}(\epsilon_k)\\
            &=
            d_{j,i} \sum_{k : x_k \in Anc_i \cup \{x_i\}} d_{i,k}^2 \mathrm{Var}(\epsilon_k), 
        \end{align*}
        \[
        \mathrm{Var}(x_i) = \sum_{k : x_k \in Anc_i \cup \{x_i\}} d_{i,k}^2 \mathrm{Var}(\epsilon_k).
        \]
        Therefore, 
        \[
        d_{j,i} = \frac{\mathrm{Cov}(x_i,x_j)}{\mathrm{Var}(x_i)}. 
        \]   
\end{compactenumi}
\end{proof}

Lastly, we quote Darmois-Skitovitch theorem (\cite{Darmois1953, Skitovich1953}) and Cram\'er's decomposition theorem \cite{Cramér}.

\begin{theorem}[Darmois-Skitovitch theorem]
    \label{thm:DS}
    Define two random variables $y_1$ and $y_2$ as linear combinations of independent random variables $w_i$, 
    $i=1,\ldots,m$:
    \[
    y_1 = \sum_{i=1}^m \alpha_i w_i, \quad 
    y_2 = \sum_{i=1}^m \beta_i w_i
    \]
    Then, if $y_1$ and $y_2$ are independent, all variables $w_j$ for which $\alpha_j\beta_j \ne 0$ are Gaussian.
\end{theorem}

\begin{theorem}[Cram\'er’s decomposition theorem]
\label{thm: Cramer}
    For two independent random variables $\epsilon_{i}$ and $\epsilon_{j}$,
    $\epsilon = \epsilon_{i} + \epsilon_{j}$ is Gaussian, if and only if $\epsilon_{i}$ and $\epsilon_{j}$ are Gaussian.
\end{theorem}

\subsection{Proofs of Theorems in Section \ref{sec: proposed alg}}
\noindent{\bf Proof of Theorem \ref{ancestor finding simple}}

Since $x_i$ and $x_j$ are adjacent in $G$, $x_i \rightarrow x_j \in E$ or $x_i \leftarrow x_j \in E$ holds. Assume that $x_i \leftarrow x_j \in E$. 
Then 
\begin{equation}
\begin{aligned}
x_i = \sum_{k : x_k \in Anc_i} d_{i,k} \epsilon_k, \\ \notag 
x_j = \sum_{k : x_k \in Anc_j} d_{j,k} \epsilon_k. 
\end{aligned}
\end{equation}
From the assumption that $x_j \sim \mathcal{NG}$, there exists $k \in Anc_j$ satisfying $\epsilon_k \sim \mathcal{NG}$.
Since $x_k \in Anc_i$, $x_i \sim \mathcal{NG}$ by the contraposition of Theorem \ref{thm: Cramer}. 
\hfill\qed\\
\bigskip

Next, we will prove Theorem \ref{ancestor finding}. 
To determine the ancestral relationship between two variables $x_i$ and $x_j$, we consider the 
pair of simple regression models (\ref{model:no ancestors}) as in Maeda and Shimizu \cite{Maeda2020}. 
We note that 
\[
\mathrm{Cov}(x_i, r_j^{(i)}) = \mathrm{Cov}(x_j, r_i^{(j)}) = 0. 
\]
Define $[p]:=\{1,\ldots,p\}$. 

Before we prove the theorem, we provide some lemmas necessary for the proof. 
\begin{lemma}
\label{lemma: v-structure}
Assume that a graph $G=\left(\bm{X}, E\right)$ satisfies the faithfulness assumption. 
For two variables $x_i\in \bm{X}$ and $x_j \in \bm{X}$, 
$x_i \indep x_j$ holds if and only if $x_i \notin Anc_j$, $x_j \notin Anc_i$, and $BCA_{ij} = \emptyset$.
% $\exists x_k, x_i \rightarrow x_k \leftarrow x_j$ or $x_i$ and $x_j$ are not connected in the causal model.
\end{lemma}
\begin{proof}\ 
\label{proof: 1}
\begin{compactenumi}
    \item Sufficiency:  
        Under the faithfulness assumption, since $x_i \indep x_j$, 
        $x_i$ and $x_j$ can be d-separated by $\emptyset$, which implies that $x_i \notin Anc_j$, $x_j \notin Anc_i$, and $BCA_{ij} = \emptyset$. 
    \item 
    Necessity: 
    If $x_i \notin Anc_j$, $x_j \notin Anc_i$, 
    then either $x_i$ and $x_j$ are disconnected, or all paths between $x_i$ and $x_j$ contain V-structures. 
    Since $BCA_{ij} = \emptyset$, $x_i$ and $x_j$ are d-separated by $\emptyset$. 
\end{compactenumi}
\end{proof}

\begin{lemma}
\label{lemma: complex chain}
For two variables $x_i$ and $x_j$, assume that the following conditions are simultaneously satisfied:
\begin{compactitem}
    \item $x_i \notindep x_j$
    \item $x_i \in Anc_j$
    \item $BCA_{ij} = \emptyset$
\end{compactitem}
If there exists $x_k \in Anc_j\cup\{x_j\}$ such that $\epsilon_k\sim \mathcal{NG}$, 
one of the following two conditions holds:
\begin{compactenumi}
    \item $(x_i, x_j\sim\mathcal{NG}) \wedge (r^{(i)}_j \indep x_i) \wedge (r^{(j)}_i\notindep x_j)$
    \item $(x_i\sim\mathcal{G}, x_j\sim\mathcal{NG}) \wedge (r^{(i)}_j \indep x_i) \wedge (r^{(j)}_i\notindep x_j)$
\end{compactenumi}
Otherwise, $(x_i, x_j\sim \mathcal{G}) \wedge (r^{(i)}_j \indep x_i) \wedge (r^{(j)}_i\indep x_j)$. 
\end{lemma}
\begin{proof}
Since $x_i \in Anc_j$, $Anc_i\cup\{x_i\} \subset Anc_j\cup\{x_j\}$ holds. 
Define disjoint sets of indices $K_A$, $K_B$ and $K_C$ by
\begin{align*}
    K_A &:= \{k \mid x_k \in Anc_i \cup \{x_i\}\},\\
    K_B &:= \{k \mid x_k \in (Anc_j \cup \{x_j\}) \setminus (Anc_i \cup \{x_i\})\}, \\
    K_C &:= [p] \setminus (K_A \cup K_B), 
\end{align*}
respectively. 
Then, using 
\begin{align*}
    x_{i} &= \sum_{k\in K_{A}}d_{i, k}\epsilon_k, \quad 
    x_{j} = \sum_{k\in K_{A}\cup K_{B}}d_{j, k}\epsilon_k, 
\end{align*}
$r^{(i)}_j$ and $r^{(j)}_i$ are expressed as 
\begin{align*}
    r^{(i)}_{j} &= x_{j} - \frac{\mathrm{Cov}(x_{i}, x_{j})}{\mathrm{Var}(x_{i})} x_{i} \\
    &= \sum_{k \in K_{A}}
    \left(d_{j, k} - \frac{\mathrm{Cov}(x_{i}, x_{j})}{\mathrm{Var}(x_{i})}d_{i, k} \right)
    \epsilon_{k} + \sum_{k\in K_{B}}d_{j,k}\epsilon_{k},\\
    r^{(j)}_{i} &= x_{i} - \frac{\mathrm{Cov}(x_{i}, x_{j})}{\mathrm{Var}(x_{j})} x_{j} \\
    &= \sum_{k\in K_{A}}
    \left(d_{i, k} - \frac{\mathrm{Cov}(x_{i}, x_{j})}{\mathrm{Var}(x_{j})}
    d_{j, k}\right)\epsilon_{k} - \sum_{k\in K_{B}}\frac{\mathrm{Cov}(x_{i}, x_{j})}{\mathrm{Var}(x_{j})}d_{j,k}\epsilon_{k}.  
\end{align*}
By $BCA_{ij} = \emptyset$, $x_{i} \in Anc_{j}$ and the faithfulness assumption, we have
\[
d_{j,i} = \frac{\mathrm{Cov}(x_i, x_j)}{\mathrm{Var}(x_i)} \neq 0
\]
and hence
\[
\frac{\mathrm{Cov}(x_i, x_j)}{\mathrm{Var}(x_j)} \ne 0.
\]
Since $d_{j,k} = d_{i,k}d_{j,i}$ for $k \in K_A$ from Lemma \ref{lem:noBCA}, $r_j^{(i)}$ and $r_i^{(j)}$ are rewritten by 
\[
r^{(i)}_{j} = \sum_{k\in K_{B}}d_{j,k}\epsilon_{k}, \quad 
r^{(j)}_{i} = \sum_{k\in K_{A}}
    \left(1-\rho_{ij}^2\right)d_{i,k}\epsilon_{k} - 
    \sum_{k\in K_{B}}
    \frac{\mathrm{Cov}(x_{i}, x_{j})}{\mathrm{Var}(x_{j})}d_{j,k}\epsilon_{k}, 
\]
where $\rho_{ij}$ is the correlation coefficient of $x_i$ and $x_j$. 
The first equality implies that $r^{(i)}_j \indep x_{i}$ always holds. 

In the case where there exists $l \in K_A$ such that $\epsilon_l \sim \mathcal{NG}$, 
both $x_i$ and $x_j$ are non-Gaussian. 
Since $(1-\rho_{ij}^2)d_{i,l} \ne 0$ from the faithfulness assumption, 
we can say that $x_j \notindep r^{(j)}_i$ by the contraposition of Darmois-Skitovich theorem. 

Consider the case where $\epsilon_k \sim \mathcal{G}$ for all $k \in K_A$ and there exists $l \in K_B$ such that $\epsilon_l \sim \mathcal{NG}$. 
Then, $x_i \sim \mathcal{G}$ and $x_j \sim \mathcal{NG}$ from the faithfulness assumption. 
Also in this case, since $d_{j,l} \ne 0$ from the faithfulness condition, we can say that $x_j \notindep r^{(j)}_i$ by the contraposition of Darmois-Skitovich theorem. 

If $\epsilon_k \sim \mathcal{G}$ for all $k \in K_A \cup K_B$, both $r^{(j)}_i$ and $x_j$ are 
Gaussian. Then, $\mathrm{Cov}(x_j,r^{(j)}_i)=0$ implies $r^{(j)}_i \indep x_j$. 
\end{proof}

\begin{lemma}
\label{lemma: complex fork}
Assume that $x_i$ and $x_j$ satisfy the following conditions. 
\begin{compactitem}
    \item $x_i \notindep x_j$
    \item $x_i \notin Anc_j, x_j \notin Anc_i$
    \item $BCA_{ij} \neq \emptyset$
\end{compactitem}
Then, $(x_i,x_j)$ generically satisfies one of the following three conditions.
\begin{compactenumi}
    \item $(x_i, x_j \sim \mathcal{NG}) \wedge  (r^{(i)}_j \notindep x_i) \wedge (r^{(j)}_i\notindep x_j)$
    \item $(x_i \sim \mathcal{G}, x_j \sim \mathcal{NG}) \wedge  (r^{(i)}_j \indep x_i) \wedge (r^{(j)}_i\notindep x_j)$
    \item $(x_i, x_j \sim \mathcal{G}) \wedge  (r^{(i)}_j \indep x_i) \wedge (r^{(j)}_i\indep x_j)$
\end{compactenumi}
\end{lemma}

\begin{proof}
Define $\overline{BCA}_{ij}$ by 
\[
\overline{BCA}_{ij} = BCA_{ij} \cup 
\left(
\bigcup_{k:x_k \in BCA_{ij}} Anc_k
\right). 
\]
In this proof, we define the four disjoint subsets of indices $K_A$, $K_B$, $K_C$, and $K_D$ as follows, 
\begin{align*}
    K_A &:= \{k \mid x_k \in \overline{BCA}_{ij}\},\\
    K_B &:= \{k \mid x_k \in Anc_i \cup \{x_i\} \setminus \overline{BCA}_{ij}\},\\
    K_C &:= \{k \mid x_k \in Anc_j \cup \{x_j\} \setminus \overline{BCA}_{ij}\},\\
    K_D &:= [p] \setminus (K_A \cup K_B \cup K_C).  
\end{align*}

Then, $x_i$, $x_j$, $r_j^{(i)}$ and $r_i^{(j)}$ are written by
\begin{align*}
    x_i &= \sum_{k \in K_A} d_{i, k} \epsilon_{k} + \sum_{k \in K_B} d_{i, k} \epsilon_k,  \\
    x_j &= \sum_{k \in K_A} d_{j, k} \epsilon_{k} + \sum_{k \in K_C} d_{j, k} \epsilon_k,  \\
    r^{(i)}_{j} &= x_j - \frac{\mathrm{Cov}(x_i, x_j)}{\mathrm{Var}(x_i)} x_i \\
    &=
    \sum_{k \in K_A}
    \left(
    d_{j, k} - \frac{\mathrm{Cov}(x_i, x_j)}{\mathrm{Var}(x_i)}d_{i, k}
    \right) \epsilon_k  
    -
    \sum_{k \in K_B} \frac{\mathrm{Cov}(x_i, x_j)}{\mathrm{Var}(x_i)} d_{i, k} \epsilon_k
    +
    \sum_{k \in K_C} d_{j, k} \epsilon_k, \\
    r^{(j)}_{i} &= x_i - \frac{\mathrm{Cov}(x_i, x_j)}{\mathrm{Var}(x_j)} x_j \\
    &=
    \sum_{k \in K_A}
    \left(
    d_{i, k} - \frac{\mathrm{Cov}(x_i, x_j)}{\mathrm{Var}(x_j)}
    d_{j, k}
    \right) \epsilon_k  
    +
    \sum_{k \in K_B} d_{i, k} \epsilon_k
    -
    \sum_{k \in K_C} \frac{\mathrm{Cov}(x_i, x_j)}{\mathrm{Var}(x_j)} d_{j, k} \epsilon_k, 
\end{align*}
respectively. 
\begin{compactenumi}
    \item[(i-a)] Suppose that there exists $l \in K_A$ such that 
        $\epsilon_{l} \sim \mathcal{NG}$. 
        Then $x_i$ and $x_j$ are non-Gaussian from the faithfulness assumption. Since 
        \begin{align}
            \label{eq:cond_nonzero}
            d_{j, l} - \frac{\mathrm{Cov}(x_i, x_j)}{\mathrm{Var}(x_i)}d_{i, l} \ne 0, 
            \quad 
            d_{i, l} - \frac{\mathrm{Cov}(x_i, x_j)}{\mathrm{Var}(x_j)}d_{j, l} \ne 0
        \end{align}
        generically holds, 
        $x_i \notindep r^{(i)}_j$ and $x_j \notindep r^{(j)}_i$ are shown by the contraposition of 
        Darmois-Skitovich Theorem. 
    \item[(i-b)] Suppose that there exist $l_B \in K_B$ and $l_C \in K_C$ such that 
        $\epsilon_{l_B}, \epsilon_{l_C} \sim  \mathcal{NG}$. 
        Then $x_i$ and $x_j$ are non-Gaussian from the faithfulness assumption. 
        Since $\mathrm{Cov}(x_i,x_j) \ne 0$ generically holds, 
        $x_i \notindep r^{(i)}_j$ and $x_j \notindep r^{(j)}_i$ are shown by the contraposition of Darmois-Skitovich Theorem. 
    \item[(ii)] Suppose that there exists $l \in K_C$ such that $\epsilon_l \sim \mathcal{NG}$ and that  
        $\epsilon_k \in \mathcal{G}$ for all $k \in K_A \cup K_B$. 
        Then $x_i \sim \mathcal{G}$ and $x_j \sim \mathcal{NG}$ from the faithfulness assumption. Since 
        $\mathrm{Cov}(x_i,r^{(i)}_j)=0$ and $\epsilon_k$ for $k \in K_A \cup K_B$ and $\epsilon_l$ for $l \in K_C$ are independent, 
        $x_i \indep r^{(i)}_j$. 
        Since $\mathrm{Cov}(x_i,x_j) \ne 0$ generically holds, 
        $x_j \notindep r^{(j)}_i$ by the contraposition of Darmois-Skitovich Theorem. 
    \item[(iii)] Suppose that $\epsilon_k \sim \mathcal{G}$ for all $k \in K_A \cup K_B \cup K_C$.
        Then $x_i, x_j \sim \mathcal{G}$ and $r^{(i)}_j, r^{(j)}_i \sim \mathcal{G}$. 
        Since $\mathrm{Cov}(x_i,r^{(i)}_j)=0$ and  
        $\mathrm{Cov}(x_j,r^{(j)}_i)=0$, 
        $x_i \indep r^{(i)}_j$ and $x_j \indep r^{(j)}_i$ hold.
\end{compactenumi}    
\end{proof}

\begin{lemma}
\label{lemma: complex triple}
Assume that $(x_i,x_j)$ satisfies the following conditions.
\begin{compactitem}
    \item $x_i \notindep x_j$
    \item $x_i \in Anc_j$
    \item $BCA_{ij} \neq \emptyset$
\end{compactitem}
Then, $(x_i,x_j)$ generically satisfies one of the following conditions. 
\begin{compactenumi}
    \item $(x_i, x_j \sim \mathcal{NG}) \wedge  (r^{(i)}_j \notindep x_i) \wedge (r^{(j)}_i\notindep x_j)$
    \item $(x_i\sim\mathcal{G}, x_j \sim \mathcal{NG}) \wedge  (r^{(i)}_j \indep x_i) \wedge (r^{(j)}_i\notindep x_j)$
    \item $(x_i \sim \mathcal{G}, x_j \sim \mathcal{G}) \wedge  (r^{(i)}_j \indep x_i) \wedge (r^{(j)}_i\indep x_j)$
\end{compactenumi}
\end{lemma}

\begin{proof}
$\overline{BCA}_{ij}$ is defined in the same way as in the proof of Lemma \ref{lemma: complex fork}. 
In this proof, we define the disjoint sets of indices $K_A$, $K_B$, and $K_C$ as follows, 
\begin{align*}
    K_A &:= \{k \mid x_k \in \overline{BCA}_{ij}\},\\
    K_B &:= \{k \mid x_k \in Anc_i \cup \{x_i\} \setminus \overline{BCA}_{ij}\},\\
    K_C &:= \{k \mid x_k \in Anc_j \cup \{x_j\} \setminus (Anc_i \cup \{x_i\})\}. 
\end{align*}
Then, $x_i$, $x_j$, $r_j^{(i)}$ and $r_i^{(j)}$ are written by
\begin{align*}
    x_i &= \sum_{k \in K_A} d_{i, k} \epsilon_{k} + \sum_{k \in K_B} d_{i, k} \epsilon_k, \\ 
    x_j &= \sum_{k \in K_A} d_{j, k} \epsilon_{k} + \sum_{k \in K_B} d_{j, k} \epsilon_{k} + \sum_{k \in K_C} d_{j, k} \epsilon_k,\\
    r^{(i)}_{j} &= x_j - \frac{\mathrm{Cov}(x_i, x_j)}{\mathrm{Var}(x_i)} x_i \\
    &=
    \sum_{k \in K_A}
    \left(
    d_{j, k} - \frac{\mathrm{Cov}(x_i, x_j)}{\mathrm{Var}(x_i)}d_{i, k}
    \right) \epsilon_k\\  
    & \qquad + 
    \sum_{k \in K_B} 
    \left(
    d_{j, k} - \frac{\mathrm{Cov}(x_i, x_j)}{\mathrm{Var}(x_i)}d_{i, k}
    \right) \epsilon_k 
    + 
    \sum_{k \in K_C} d_{j, k} \epsilon_k,\\
    r^{(j)}_{i} 
    =& 
    x_i - \frac{\mathrm{Cov}(x_i, x_j)}{\mathrm{Var}(x_j)} x_j \\
    =&
    \sum_{k \in K_A}
    \left(
    d_{i, k} - \frac{\mathrm{Cov}(x_i, x_j)}{\mathrm{Var}(x_j)}d_{j, k}
    \right) \epsilon_k\\
    &\qquad +
    \sum_{k \in K_B} \left(
    d_{i, k} - \frac{\mathrm{Cov}(x_i, x_j)}{\mathrm{Var}(x_j)}d_{j, k}
    \right) \epsilon_k
    -
    \frac{\mathrm{Cov}(x_i, x_j)}{\mathrm{Var}(x_j)} \sum_{k \in K_C} d_{j, k} \epsilon_k.
\end{align*}
\begin{compactenumi}
    \item[(i-a)] Suppose that there exists $l \in K_A$ such that
    $\epsilon_l \sim \mathcal{NG}$.
    Then $x_i,x_j \sim \mathcal{NG}$ by the faithfulness assumption. 
    Since (\ref{eq:cond_nonzero})
    % \[
    % d_{j, l} - \frac{\mathrm{Cov}(x_i, x_j)}{\mathrm{Var}(x_i)}d_{i, l} \ne 0, \quad 
    % d_{i, l} - \frac{\mathrm{Cov}(x_i, x_j)}{\mathrm{Var}(x_j)}d_{j, l} \ne 0
    % \]
    generically holds, 
    $r^{(i)}_j \notindep x_i$ and $r^{(j)}_i \notindep x_j$ are shown 
    by the contraposition of Darmois-Skitovich Theorem. 
    \item[(i-b)] Suppose that there exists $l \in K_B$ such that $\epsilon_{l}$. 
    Then $x_i$ and $x_j$ are non-Gaussian from the faithfulness assumption.  
    Since (\ref{eq:cond_nonzero})
    % \[
    % d_{j, l} \ne \frac{\mathrm{Cov}(x_i, x_j)}{\mathrm{Var}(x_i)}d_{i, l}, \quad 
    % d_{i, l} \ne \frac{\mathrm{Cov}(x_i, x_j)}{\mathrm{Var}(x_j)}d_{j, l}
    % \]
    generically holds, 
    $r^{(i)}_j \notindep x_i$ and $r^{(j)}_i \notindep x_j$ are shown 
    by the contraposition of Darmois-Skitovich Theorem. 
    \item[(ii)] Suppose that there exists $l \in K_C$ such that $\epsilon_l \sim \mathcal{NG}$ and 
    that $\epsilon_k \sim \mathcal{G}$ for all $k \in K_A \cup K_B$. 
    Then $x_i \sim \mathcal{G}$ and $x_j \sim \mathcal{NG}$ from the faithfulness assumption. 
    Since $\mathrm{Cov}(x_i,r^{(i)}_j)=0$ and $\epsilon_k \indep \epsilon_l$ for all $k \in K_A \cup K_B$ and $l \in K_C$, 
    we have $x_i \indep r^{(i)}_j$. 
    Since $\mathrm{Cov}(x_i,x_j) \ne 0$ generically holds, 
    $r^{(j)}_i \notindep x_j$
    by the contraposition of the Darmois-Skitovich Theorem. 
    \item[(iii)] Suppose that $\epsilon_k \sim \mathcal{G}$ for all $k \in K_A \cup K_B \cup K_C$.
    Then $x_i, x_j \sim \mathcal{G}$ and $r^{(i)}_j, r^{(j)}_i \sim \mathcal{G}$. 
    Hence $\mathrm{Cov}(x_i,r^{(i)}_j)=0$ and  
    $\mathrm{Cov}(x_j,r^{(j)}_i)=0$ imply 
    $x_i \indep r^{(i)}_j$ and $x_j \indep r^{(j)}_i$.
\end{compactenumi}
\end{proof}
\noindent{\bf Proof of Theorem \ref{ancestor finding}}

Based on the lemmas mentioned above, we summarize the results and proofs as follows.
\begin{compactenumi}
    \item In the case of $x_i \indep x_j$, we can conclude that $x_i \notin Anc_j$ and $x_j \notin Anc_i$ from Lemma \ref{lemma: v-structure}.  
    \item From Lemma \ref{lemma: complex chain}, 
    $(x_i, x_j \sim \mathcal{NG}) \wedge  (r^{(i)}_j \indep x_i) \wedge (r^{(j)}_i\notindep x_j)$ implies that 
    $x_i \in Anc_j$ and $BCA_{ij}=\emptyset$. 
    \item From Lemma \ref{lemma: complex fork} to \ref{lemma: complex triple},  
    $(x_i, x_j \sim \mathcal{NG}) \wedge  (r^{(i)}_j \notindep x_i) \wedge (r^{(j)}_i\notindep x_j)$ implies that $BCA_{ij} \ne \emptyset$.      
\end{compactenumi}
\hfill \qed

\bigskip

\noindent{\bf Proof of Theorem \ref{thm:identifiability}}

Given that the PC-LiNGAM can identify up to a DEP, it suffices to show that the proposed Algorithm \ref{alg: proposed} can identify the orientation of undirected edges in a DSEP containing nodes with non-Gaussian disturbance. 
Suppose that $x_i \rightarrow x_j \in E$ and $x_i - x_j \in E_{\mathrm{ud}}$.    

If $\epsilon_i, \epsilon_j \sim \mathcal{NG}$, $x_i$ and $x_j$ are also non-Gaussian. 
If $BCA_{ij} = \emptyset$,  
$r^{(i)}_j \indep x_i$ and $r^{(j)}_i \notindep x_j$ generically hold from Lemma \ref{lemma: complex chain}. 
Therefore, from (ii) in Theorem \ref{ancestor finding simple}, Algorithm \ref{alg: proposed} can identify $x_i \in Anc_j$. 
If $BCA_{ij} \ne \emptyset$,  $r^{(i)}_j \notindep x_i$ and $r^{(j)}_i \notindep x_j$ generically holds from Lemma \ref{lemma: v-structure} and \ref{lemma: complex chain}. 
Therefore, from (iii) in Theorem \ref{ancestor finding simple}, we can identify $BCA_{ij} \ne \emptyset$. 

Suppose that $\epsilon_i \sim \mathcal{G}$ and $\epsilon_j \sim \mathcal{NG}$. 
Then $x_j$ is non-Gaussian, and $x_i$ could be Gaussian or non-Gaussian. 
If $x_i$ is non-Gaussian, we can show that Algorithm \ref{alg: proposed} can generically identify $x_i \in Anc_j$ or $BCA_{ij} \ne \emptyset$ in the same way as in the above argument. 
If $x_i$ is Gaussian, we conclude that $x_i \in Anc_j$ from Corollary \ref{cor: ancestor finding simple}. 

Suppose that $\epsilon_i \sim \mathcal{NG}$ and $\epsilon_j \sim \mathcal{G}$. Since $x_i \in Anc_j$ in the true causal graph, $x_j$ has to be non-Gaussian. Hence, $x_i$ and $x_j$ are non-Gaussian. 
Therefore, we can show that Algorithm \ref{alg: proposed} can generically identify $x_i \in Anc_j$ or $BCA_{ij} \ne \emptyset$ in the same way as in the discussion above. 
\hfill \qed

\subsection{Some properties of an undirected subgraph of a DSEP}
\label{sec:DMG}
In this subsection, we summarize some properties of connected components of $G_{\mathrm{ud}}$. 
We first define a directed moral graph. 
\begin{definition}
\label{def: DMG}
Let $G=\left(\bm{X},E\right)$ be a weakly connected DAG. 
If $G$ satisfies either of the following conditions, 
\begin{compactenumi}
    \item $|Pa_k| \le 1$ for $k=1,\ldots,p$
    \item For any $x_k \in \bm{X}$ such that $|Pa_k| \ge 2$ and for any pair $x_i,x_j \in Pa_k$, 
    $x_i \to x_j \in E$ or $x_j \to x_i \in E$ holds,
\end{compactenumi}
we call $G$ a directed moral graph (DMG). 
%For any $x_k \in \bm{X}$ such that $|Pa_k| \ge 2$ and for any pair $x_i,x_j \in Pa_k$, 
%if $x_i \to x_j \in E$ or $x_j \to x_i \in E$, we call $G$ a directed moral graph (DMG). 
\end{definition}

After obtaining a DSEP using the PC algorithm, the proposed method orients undirected edges according to 
Theorem \ref{ancestor finding simple BCA star} and Theorem \ref{ancestor finding BCA star}. 
Then, we need to focus on the undirected induced subgraphs $G_{\mathrm{ud}}=(\bm{X}_{\mathrm{ud}}, E_{\mathrm{ud}})$. In general, $G_{\mathrm{ud}}$ is not connected. 
Then, we have the following theorem. 

\begin{theorem}
    Every weakly connected component of induced subgraph $G(\bm{X}_{\mathrm{ud}})$ is a DMG. 
\end{theorem}
\begin{proof}
    Assume that there exists a weakly connected component $G'=(\bm{X}',E')$ of $G(\bm{X}_{\mathrm{ud}})$ that is not a DMG. 
    Then there exists $x_k \in \bm{X}'$ such that 
    %there exist $ x_i, x_j \in Pa_k$, 
    $x_i \to x_j \notin E^{\prime}_{\mathrm{di}} $ and $ x_j \to x_i \notin E^{\prime}_{\mathrm{di}}$ hold
    for some $x_i,x_j \in Pa_k$. 
    This implies that there exists $ \bm{X}''$ such that $x_k \notin \bm{X}''$, satisfying 
    $x_i \indep x_j \mid \bm{X}''$. Hence, the V-structure $x_i \rightarrow x_k \leftarrow x_j$ is detected by the PC algorithm. 
\end{proof}

By definition, the DSEP of a DMG is inherently undirected. 

\begin{theorem}
\label{thm: remove one node}
    Let $G'=\left(\bm{X},E\right)$ be a DMG. 
    Every weakly connected induced subgraph $G'(\bm{X}')=(\bm{X}',E')$ for $\bm{X}' \subset \bm{X}$ is also a DMG.
\end{theorem}
\begin{proof}
Assume that there exists an induced subgraph $G'(\bm{X}')$ of $G'$ that satisfies weak connectivity but is not a DMG. Then, from the definition of a DMG, 
\[
\exists x_{k} \in \bm{X}',\; \exists x_{i},x_{j} \in Pa_{k} \text{~~s.t.~~} 
x_{i} \to x_{j} \notin E' \wedge x_{j} \to x_{i} \notin E'.
\]
Since $G'(\bm{X}')$ is an induced subgraph of $G'$, %we have $\bm{X}' \subset \bm{X}$ and $E' \subset E$. 
%Therefore, 
$x_{i} \to x_{j} \notin E \wedge x_{j} \to x_{i} \notin E$. 
\end{proof}

\begin{theorem}
\label{thm: one source}
    Any DMG has only one source node.
\end{theorem}
\begin{proof}
    We prove the theorem by induction on the number of nodes $p$. 
    The theorem is trivial when $p=1$ and $p=2$. 
    Assume that the theorem holds for any DMG with $p\ge3$ nodes. 

    Let $G'=(\bm{X}',E')$ be a DMG with $p+1$ nodes. 
    Assume that $G'$ has at least two different source nodes $x_0$ and $x_1$. 
    By Theorem \ref{thm: remove one node} and the inductive assumption, each connected component of the induced subgraph $G'(\bm{X}' \setminus \{x_0\})$ is a DMG, and therefore each has only one source node. 
    Let $G''=(\bm{X}'',E'')$ be the connected component whose only source node is $x_1$. 
    Let $ch_0$ be the set of children of $x_0$ in $G'$.
    For all $x_k \in ch_0 \cap \bm{X}''\ne\emptyset$, $x_1 \in Anc_k$ in $G'$. 
    Therefore there exists $x_k \in ch_0 \cap \bm{X}''$ and $x_l \in Pa_k$ in $G'$ satisfying 
    $x_0 \to x_l \notin E'$ and $x_0 \leftarrow x_l \notin E'$, which contradicts the assumption that 
    $G'$ is a DMG. 
\end{proof}

The following Theorem \ref{thm: DMG equivalent} helps in understanding the procedure of Algorithm \ref{alg: exception}. 

\begin{theorem}
\label{thm: DMG equivalent}
Assume that $G$ is a weakly connected DAG. The following three conditions are equivalent.
\begin{compactenumi}
    \item $G$ is a DMG.
    \item Every weakly connected induced subgraph of $G$ is a DMG.
    \item Every weakly connected induced subgraph of $G$ contains one source node.
\end{compactenumi}
\end{theorem}

\begin{proof}
(i) $\Rightarrow$ (ii) and (ii) $\Rightarrow$ (iii) are established precisely by Theorem \ref{thm: remove one node} and \ref{thm: one source}, respectively. It suffices to show (iii) $\Rightarrow$ (i). If $G$ satisfies (iii) but fails to satisfy (i), then $\exists x_{k},\exists x_{i},x_{j}\in Pa_{k}$ such that no directed edge exists between $x_{i}$ and $x_{j}$. Therefore, the connected induced subgraph $x_{i}\to x_{k} \gets x_{j}$ will contain $x_{i}$ and $x_{j}$ as two source nodes, contradicting the assumption that $G$ satisfies (iii). 
\end{proof}

\end{document}